\documentclass{article}

\PassOptionsToPackage{numbers,compress}{natbib}
\usepackage[preprint]{neurips_2026}

\usepackage[utf8]{inputenc}
\usepackage[T1]{fontenc}
\usepackage{hyperref}
\usepackage{url}
\usepackage{booktabs}
\usepackage{amsfonts}
\usepackage{nicefrac}
\usepackage{microtype}
\usepackage{xcolor}

\usepackage{enumitem}
\usepackage{graphicx}
\usepackage{subcaption}
\usepackage{amsmath}
\usepackage{amssymb}
\usepackage{amsthm}
\usepackage{cleveref}
\usepackage[table]{xcolor}

\theoremstyle{remark}

\theoremstyle{plain}

\newcounter{obs}
\crefname{obs}{Observation}{Observations}
\Crefname{obs}{Observation}{Observations}
\newcommand{\obshead}{%
  \phantomsection
  \refstepcounter{obs}%
  \paragraph*{Observation \theobs}%
}

\newtheorem{proposition}{Proposition}[section]
\theoremstyle{remark}

\usepackage{makecell}
\usepackage{wrapfig}
\usepackage{placeins}
\usepackage{float}
\usepackage[most]{tcolorbox}

\tcolorboxenvironment{proposition}{
  enhanced,
  breakable,
  colback=black!3,
  colframe=black!25,
  boxrule=0.7pt,
  arc=1.5mm,
  left=7pt,
  right=7pt,
  top=6pt,
  bottom=6pt,
  before skip=10pt,
  after skip=10pt
}

\definecolor{linkblue}{HTML}{1A4FB5}
\hypersetup{
    colorlinks=true,
    linkcolor=linkblue,
    citecolor=linkblue,
    urlcolor=linkblue
}

\newcommand{\R}{\mathbb{R}}

\title{The Seriality Gap in Video Diffusion Models}

\author{%
  Jorge Diaz Chao\thanks{Equal contribution.} \quad
  Konpat Preechakul\footnotemark[1] \quad
  Yuxi Liu \quad
  Yutong Bai \\\\
  UC Berkeley \\
  \texttt{\{jdiazchao,konpat,yuxi\_liu,yutongbai\}@berkeley.edu}
}

\begin{document}

\maketitle
\begingroup
\renewcommand{\thefootnote}{}
\footnotetext{\url{https://seriality-gap.jdiazchao.com}}
\addtocounter{footnote}{-1}
\endgroup

\begin{figure}[h]
  \centering
  \includegraphics[width=0.95\linewidth]{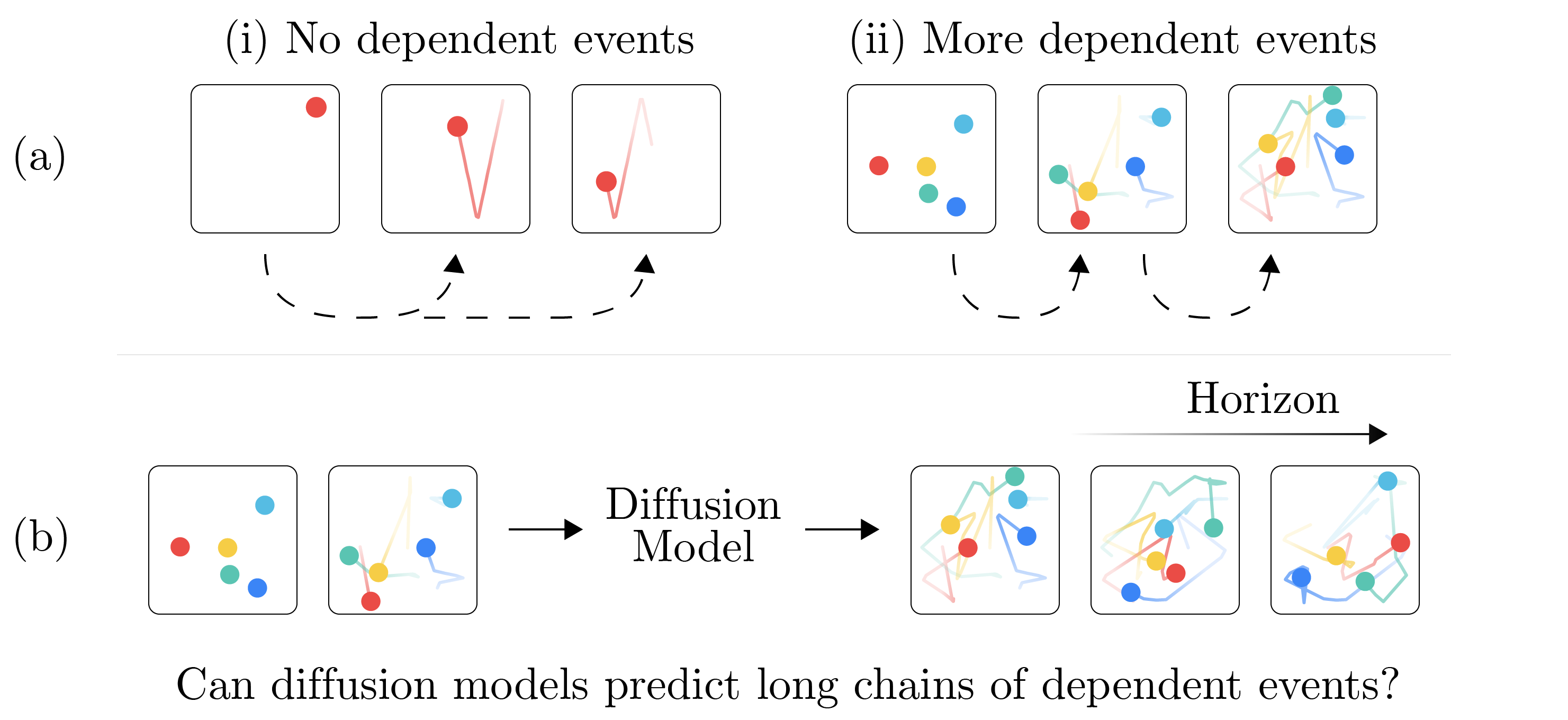}
  \caption{\textbf{Dependent-event prediction exposes the seriality gap.}
  (a) Hard-sphere dynamics separate non-serial from serial video prediction.
  (i) In the single-ball control, any future state can be computed directly from the initial state, without resolving intermediate states.
  (ii) With multiple balls, each ball--ball collision changes the state governing later collisions, creating dependent-event chains that must be resolved in temporal order.
  (b) Given initial frames, can a diffusion model remain accurate as longer prediction horizons demand more serial computation?}
  \label{fig:teaser}
\end{figure}

\begin{abstract}

When one ball strikes another, then another, video models should predict the consequences of each bounce.
In controlled experiments on multi-ball hard-sphere dynamics, we find that the performance of standard bidirectional video diffusion degrades as the causal chain lengthens, even when provided more denoising steps.
In a length-matched single-ball control, where ball--ball interactions are absent, the degradation largely disappears, isolating dependent-event structure rather than video length as the cause.
Across intervention studies, \textbf{methods that increase effective serial computation improve performance disproportionately}, including autoregressive/blockwise generation and architectural depth.
We identify this pattern as the \textbf{seriality gap}: a mismatch between tasks requiring growing serial computation and video diffusion models whose denoising loop does not provide scalable serial compute.
We then prove that, for deterministic video prediction, denoising steps do not add serial computation beyond the backbone, indicating a structural obstacle for video diffusion on serial reasoning and simulation tasks.

\end{abstract}

\newpage

\section{Introduction}
\label{sec:intro}

When one ball strikes a second, which then strikes a third, predicting the chain of consequences seems like a minimal competence for any video model.
Yet, even as video diffusion models generate compelling clips and are increasingly positioned as world simulators and visual reasoners \citep{videoworldsimulators2024,wiedemer2025videomodelszeroshotlearners}, persistent failures in physical consistency \citep{meng2024worldsimulatorcraftingphysical,guo2025t2vphysbenchfirstprinciplesbenchmarkphysical,motamed2025generativevideomodelsunderstand,raedsch2026physicsiqverified,tragoudaras2026evaluatingnewtonianmechanicsvideo,Huang_2024_CVPR,zheng2025vbench20advancingvideogeneration} and reasoning \citep{wiedemer2025videomodelszeroshotlearners,newman2026videomodelsreasonearly} remain.
What if this is due not only to a lack of data coverage but also to a limitation of the diffusion model itself? We therefore ask the following:

\begin{center}
\emph{Can video diffusion models predict ever-longer chains of dependent events?}
\end{center}

This simple question is consequential precisely because the \textit{negative} answer would be surprising.
Because diffusion generates videos through repeated denoising, failure on this task would call into question what computation those iterations actually provide---an issue that has recently received substantial attention but remains unresolved.
Some studies attribute reasoning to computation across denoising steps \citep{wang2026demystifying}; others report erosion of motion structure over later steps \citep{han2026phaselock}, early plan commitment and the need for chained generations for longer solutions \citep{newman2026videomodelsreasonearly}, or rapidly saturating gains from additional steps \citep{ma2025inferencetimescalingdiffusionmodels}.
Together, these findings motivate characterizing whether denoising supplies computation that scales with dependency-chain length (Section~\ref{sec:related_work}).

We study this question with a deliberately minimal video-prediction task: \textbf{hard-sphere dynamics}.
Identical balls move frictionlessly in a box and undergo elastic collisions; given the initial frames of a trajectory, the model must predict the future video (Figure~\ref{fig:teaser}b).
Here, visual complexity is minimal, the dynamics are deterministic, and prediction errors are unambiguous.
Our setup varies the length of the dependent-event chain: with $n>1$ balls, longer $f$-frame rollouts contain longer chains of ball--ball interactions, whereas $n=1$ acts as a control with no such interactions and admits a closed-form solution.
This lets us vary serial complexity while holding visual content essentially fixed, a control difficult to obtain in natural video (Section~\ref{sec:benchmark}).

The task's simplicity is deliberate.
Natural video adds appearance complexity, uncertainty, and ambiguity, but none removes the need to propagate dependent events.
Failure in this controlled setting therefore identifies a bottleneck that may persist in richer settings; conversely, success would address that bottleneck, not solve video modeling as a whole.

\vspace{-0.5em}
\begin{figure}[h]
  \centering
  \includegraphics[width=0.85\linewidth]{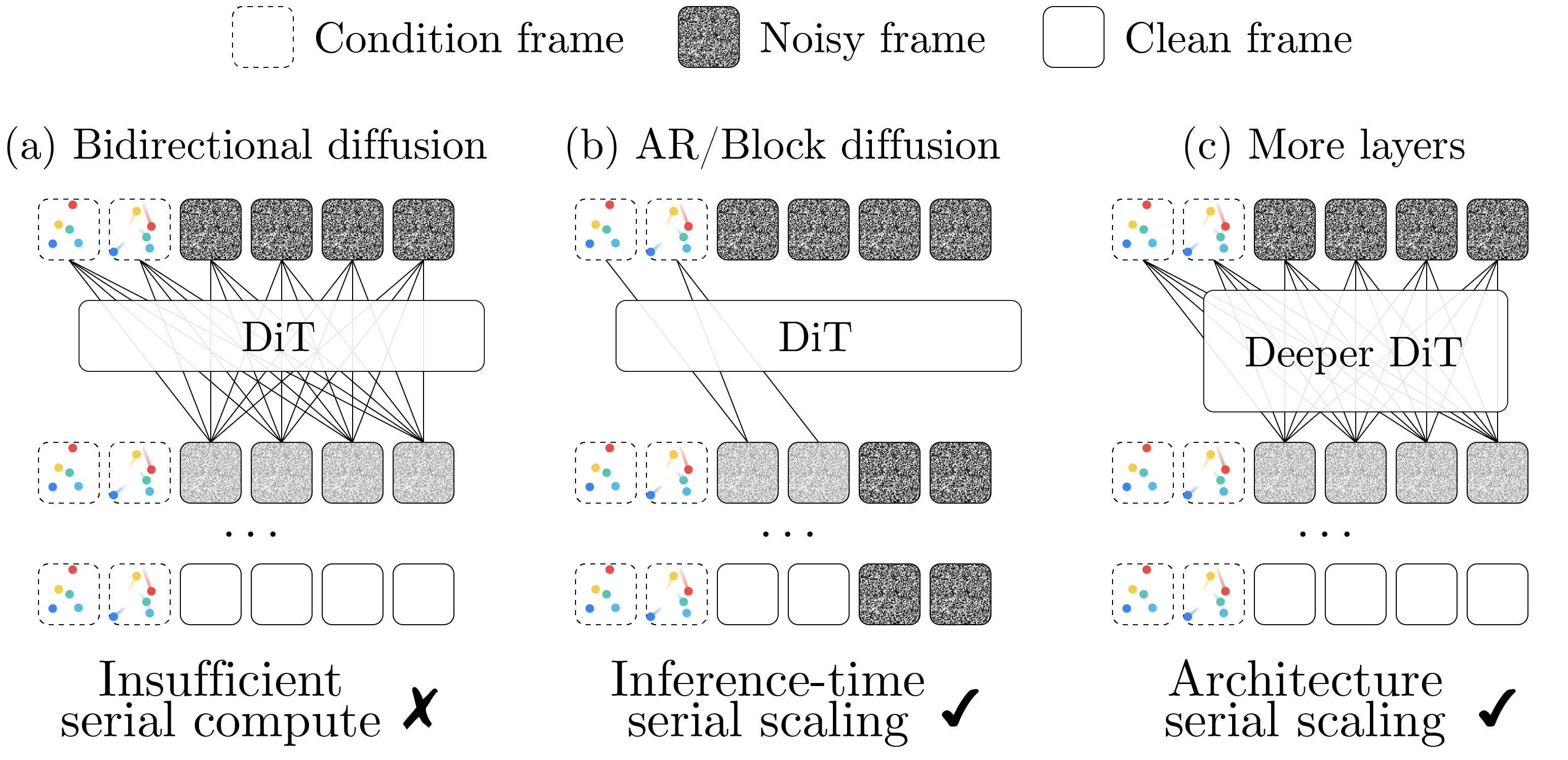}
  \caption{\textbf{Serial scaling.}
  (a) Bidirectional diffusion jointly denoises future frames, but denoising does not supply scalable serial compute (Section~\ref{theory}).
  Consistent with our theory, performance improves with
  (b) autoregressive/blockwise inference factorization or
  (c) a deeper backbone.}
  \label{fig:serial_scaling}
\end{figure}
\vspace{-1.0em}

\paragraph{The seriality gap (Section~\ref{results})}
In controlled multi-ball ($n=5$) experiments, standard \textit{bidirectional} video diffusion degrades as causal chains lengthen.
This degradation largely disappears in the length-matched single-ball ($n=1$) control, showing that the gap is driven by more difficult event dependencies rather than video length.
Increasing denoising steps provides little benefit and does not close the gap.
By contrast, at comparable total compute, interventions that add effective serial computation help disproportionately: increasing depth is more effective than increasing width, and serializing generation yields monotonic improvements---from bidirectional generation through progressively smaller temporal blocks to full autoregression.

\paragraph{Why denoising steps do not close the gap (Section~\ref{theory})}
The result is counterintuitive for two reasons.
First, longer rollouts already give bidirectional video DiTs~\cite{Peebles_2023_ICCV} far more total compute than even a simulator would require: temporal self-attention scales as $\mathcal{O}(f^2)$, while a physical rollout needs only $\mathcal{O}(f)$.
Second, diffusion denoising is itself iterative, suggesting that more steps could close the seriality gap.
Both intuitions fail.
The issue is not total computation but scalable serial computation.

For a closed, deterministic video process, the future is uniquely determined by the initial frames. If the score network computes the exact score, a single evaluation contains enough information to recover that future (Proposition~\ref{prop1}). Repeated denoising therefore does not provide additional serial computation beyond what the backbone has already computed. 
When predicting dependent events requires serial computation beyond the capacity of a fixed-depth backbone, the required computation must instead come from greater backbone depth or an inference factorization—such as autoregression—aligned with the temporal dependency graph (Figure~\ref{fig:serial_scaling}).
This explains why depth and temporal factorization help to predict dependent events, while additional denoising steps do not.

\paragraph{Beyond hard-sphere dynamics}
Under accurate score matching, denoising is not an independent source of scalable serial computation beyond the backbone. This limitation is not specific to hard-sphere dynamics: it can arise whenever diffusion models face tasks whose serial demands grow with problem size, including maze and puzzle solving. In such tasks, however, the dependency structure need not follow temporal order, so frame-wise autoregression is not a universal remedy. The needed serial computation must instead come from greater backbone depth, a task-aligned inference factorization, or potentially a new class of non-score-matching generative models.

\vspace{-0.5em}\paragraph{Contributions}
\begin{enumerate}[label=(\roman*), leftmargin=*, topsep=0.25em, itemsep=0.25em]
    \item We introduce a video-prediction testbed based on hard-sphere dynamics for studying \emph{dependent-event prediction}. The testbed varies serial complexity while holding visual content essentially fixed (Section~\ref{sec:benchmark}).
    \item We identify the \emph{seriality gap}: the performance of bidirectional video diffusion degrades as chains of dependent events lengthen, despite increasing total compute. This degradation is not explained by video length or closed by additional denoising steps, but is reduced by interventions that add effective serial computation (Section~\ref{results}).
    \item We show theoretically that, for deterministic video prediction, denoising steps do not add scalable serial computation beyond what is already computable by the backbone (Section~\ref{theory}).
\end{enumerate}

\section{Hard-Sphere Dynamics as a Testbed for Dependent-Event Prediction}
\label{sec:benchmark}

In order to study serial temporal prediction, we design a setting where dependency length can vary while visual complexity and ambiguity remain controlled.
We use \textbf{hard-sphere dynamics} (Figure~\ref{fig:teaser}a): $n$ identical balls move frictionlessly in a square box and undergo perfectly elastic collisions with walls and with one another.
The rendering is simple, the dynamics are deterministic, and each initial condition has a single ground-truth rollout.
This deliberately minimal setting is diagnostic---if a video model fails here, the failure is unlikely to be explained by visual complexity or ambiguous futures.

In this setting, serial dependence is determined by the number of ball--ball collisions.
This lets us compare multi-ball ($n>1$) rollouts, where collision chains grow with the prediction horizon, against single-ball ($n=1$) rollouts, which can be equally long but have closed-form, non-serial futures. This contrast separates difficulty due to video length from difficulty due to serial dependence.

\subsection{Experimental Setting}
\label{sec:experiment}

\paragraph{Dataset}
\label{par:dataset}

For each dataset configuration, we synthesize $f$-frame hard-sphere simulation videos at $128 \times 128$ resolution.
Because the simulator is deterministic, each initial condition has a unique ground-truth rollout, so prediction error can be measured directly and unambiguously.
Appendix~\ref{data-synth} documents the simulator, initial-condition sampling, and filtering criteria.

We construct two dataset families.
In the serial setting ($n=5$), our filtering procedure enforces a minimum post-conditioning ball--ball collision density, so the minimum number of collision events grows with the $f$-frame prediction horizon.
Since each collision can change the state that determines subsequent collisions, these events must be resolved in temporal order, making $f$ a proxy for serial complexity.
In the length-matched control ($n=1$), ball--ball collisions are absent, so any future state can be computed in closed form from the initial conditions without simulating intermediate states.

As $f$ increases, bidirectional video diffusion models process larger spatiotemporal tensors and therefore receive more total computation~\citep{wan2025wanopenadvancedlargescale}.
If this computation were sufficient for tracking dependent events, local physical accuracy in the multi-ball setting should not degrade substantially more with $f$ than in the single-ball control.

\paragraph{Models}
\label{par:models}

We train DiT-style~\citep{Peebles_2023_ICCV} video diffusion models adapted from \texttt{Wan2.1}~\citep{wan2025wanopenadvancedlargescale}, a standard bidirectional video diffusion architecture with competitive performance.
Every model operates in the \texttt{Wan2.1} VAE latent space, which downsamples videos by $8\times$ along each spatial axis and $4\times$ temporally. The patch embedder groups $2\times2$ spatial latent patches without additional temporal downsampling, so one temporal latent frame corresponds to four rendered video frames.

We compare standard bidirectional denoising with variants that allocate computation more serially, as summarized in Figure~\ref{fig:serial_scaling}.
Autoregressive and block-autoregressive variants use the same base architecture but replace full temporal attention with FlexAttention~\citep{dong2024flex} block-causal masks over temporal latent frames.
We write Block-$k$ for generation in blocks of $k$ temporal latent frames; Appendix~\ref{attention-masks} shows the training and inference attention masks.
Unless otherwise specified, evaluation uses $T=50$ denoising steps.
To separate inference factorization from architecture, we also vary backbone width $d_\text{model}$ and depth $n_\text{layers}$.

\paragraph{Training}
\label{training}
Models are trained for $200{,}000$ iterations with a batch size of $64$ using AdamW with a weight decay of $0.01$ and a constant learning rate of $2 \times 10^{-4}$. We maintain an exponential moving average (EMA) of the weights with decay $0.9999$ and use the EMA weights for evaluation. The reported training runs require approximately $300$ A100-80GB GPU-days.

\paragraph{Evaluation}
\label{eval}
We evaluate predictions along three axes. 

\begin{enumerate}[leftmargin=*, topsep=0.25em, itemsep=0.25em, parsep=0pt]
    \item \textbf{Global error $\Delta x^{(\mathrm{GT})}$} measures the average Euclidean distance between predicted and ground-truth ball centers across all balls and evaluated frames. This is the most direct notion of prediction accuracy, but it is sensitive to accumulated long-horizon drift.

    \item \textbf{Local error $\Delta x^{(k)}$} measures local physical consistency. For each generated frame, we go back $k$ frames, extract ball centers and velocities from the generated pixels, roll the simulator forward for $k$ steps, and compare the result to the generated frame. This avoids penalizing accumulated drift caused by state-estimation errors. By default, we report $\Delta x^{(5)}$.

    \item \textbf{Visual quality (IoU)} measures shape fidelity independently of position error. For each ball, we center-align its generated silhouette with an ideal circular mask of the correct radius and compute their intersection over union.
\end{enumerate}

Let $\mathcal{I}=\{f_0+1,\ldots,f\}$ denote the predicted frames used for evaluation, with $f_0=5$ conditioning frames in all experiments. We define the global and local accuracy metrics as follows:
\begin{equation*}
\Delta x^{(\mathrm{GT})}
=
\frac{1}{|\mathcal{I}|\,n}
\sum_{t \in \mathcal{I}} \sum_{i=1}^{n}
\left\| \hat{x}_{t,i} - x^{\mathrm{GT}}_{t,i} \right\|_2,
\end{equation*}
\begin{equation*}
  \Delta x^{(k)}
=
\frac{1}{|\mathcal{I}|\,n}
\sum_{t \in \mathcal{I}} \sum_{i=1}^{n}
\left\| \hat{x}_{t,i} - \tilde{x}^{(k)}_{t,i} \right\|_2,
\end{equation*}
where $\hat{x}_{t,i}$ denotes the predicted position of ball $i$ at time $t$, $x^{\mathrm{GT}}_{t,i}$ the ground-truth position, and $\tilde{x}^{(k)}_{t,i}$ the position obtained by extracting the generated state at time $t-k$ and rolling the simulator forward $k$ steps. A human evaluation in Appendix~\ref{app:human-eval} confirms that large gaps in these metrics are salient.

Since models predict pixels rather than coordinates, we recover $\hat{x}_{t,i}$ by thresholding the known color range of ball $i$ in each generated frame and taking the centroid of the segmented mask; $x^{\mathrm{GT}}_{t,i}$ is obtained from the simulator.

\pagebreak

\section{The Seriality Gap in Hard-Sphere Dynamics}
\label{results}

We summarize the empirical findings as five observations. Appendix~\ref{app:numerical-results} gives full results; Appendix~\ref{app:alternative-explanations} rules out optimization and initialization confounds.

\begin{figure}[h]
    \centering
    
    \begin{subfigure}[t]{0.48\linewidth}
        \centering
        \includegraphics[width=\linewidth]{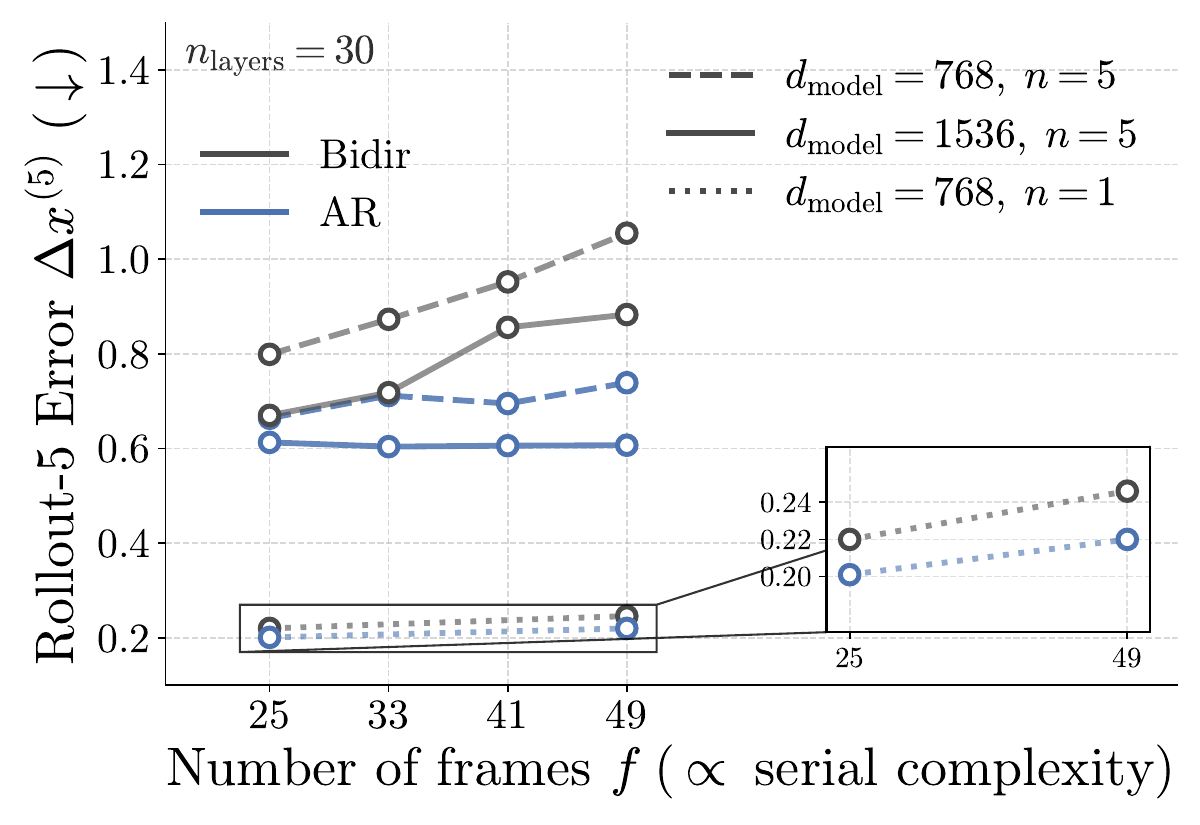}
    \end{subfigure}
    \hfill
    \begin{subfigure}[t]{0.48\linewidth}
        \centering
        \includegraphics[width=\linewidth]{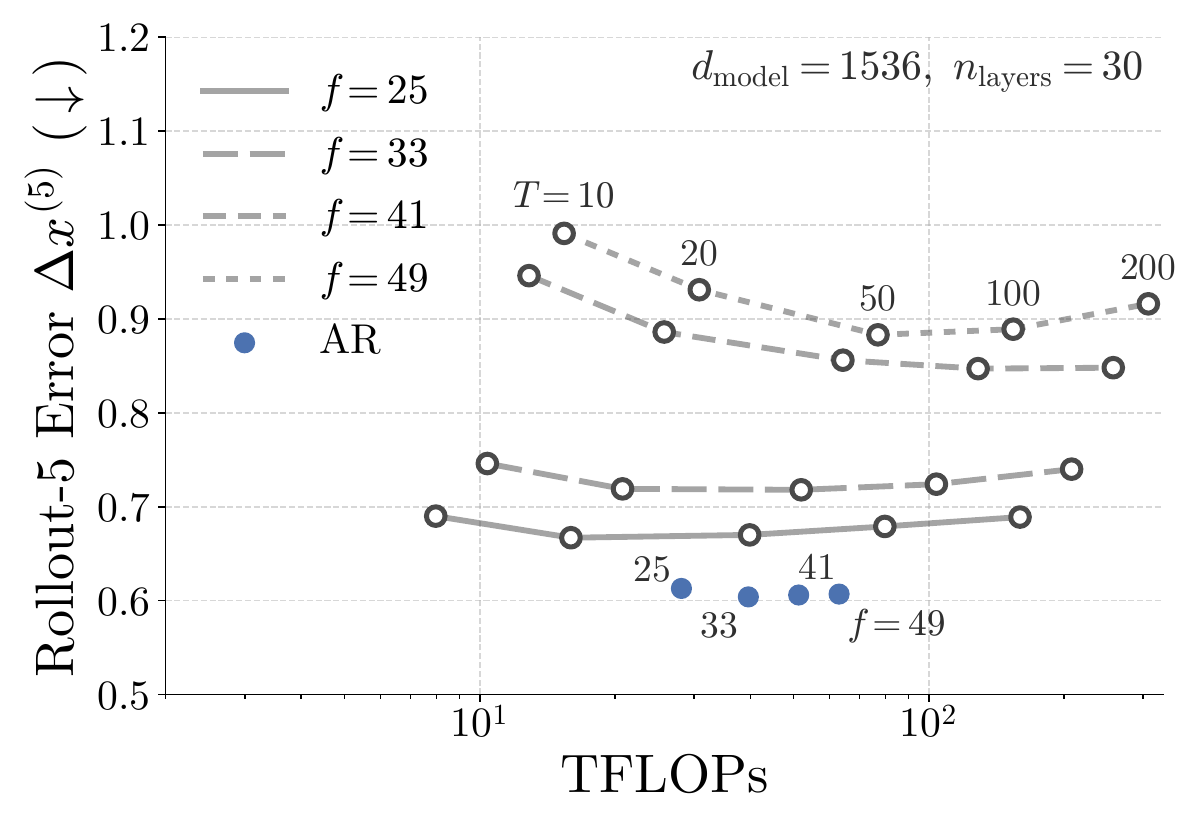}
    \end{subfigure}
    
    \vspace{0.5em}
    
    \begin{subfigure}[t]{0.48\linewidth}
        \centering
        \includegraphics[width=\linewidth]{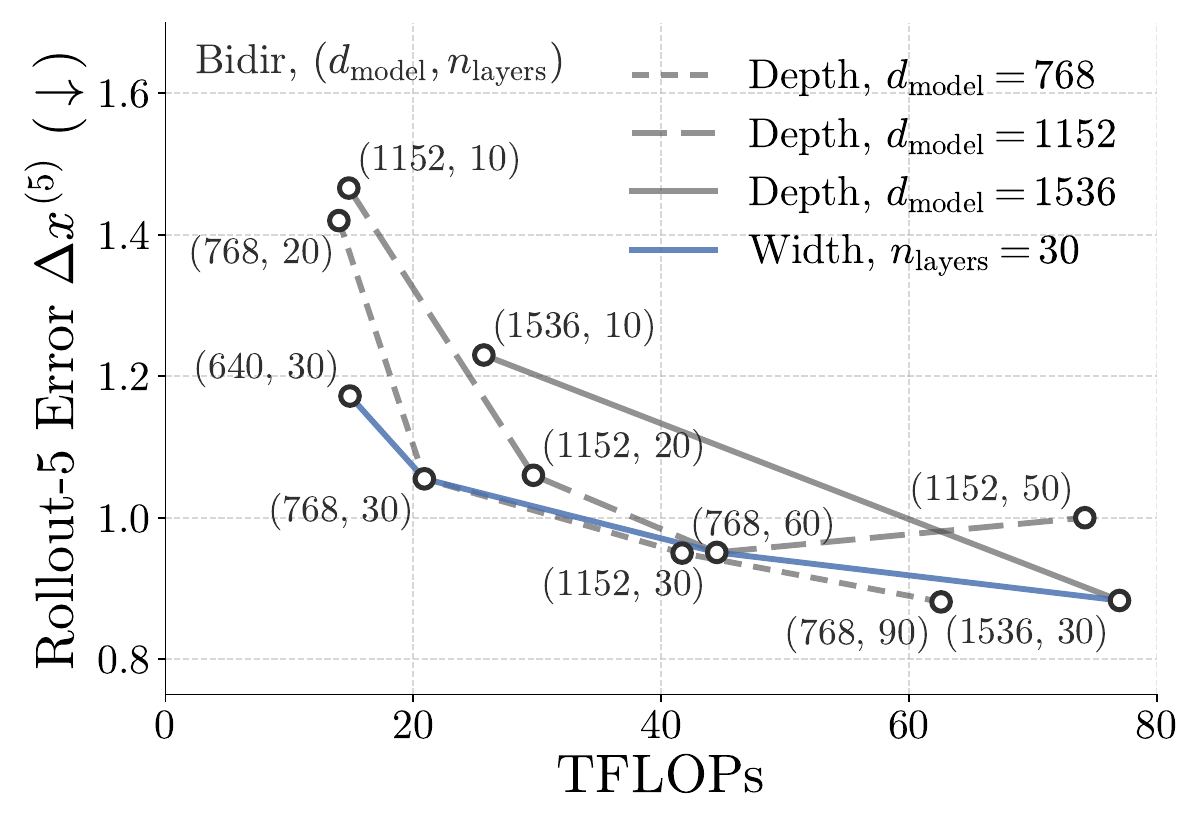}
    \end{subfigure}
    \hfill
    \begin{subfigure}[t]{0.48\linewidth}
        \centering
        \includegraphics[width=\linewidth]{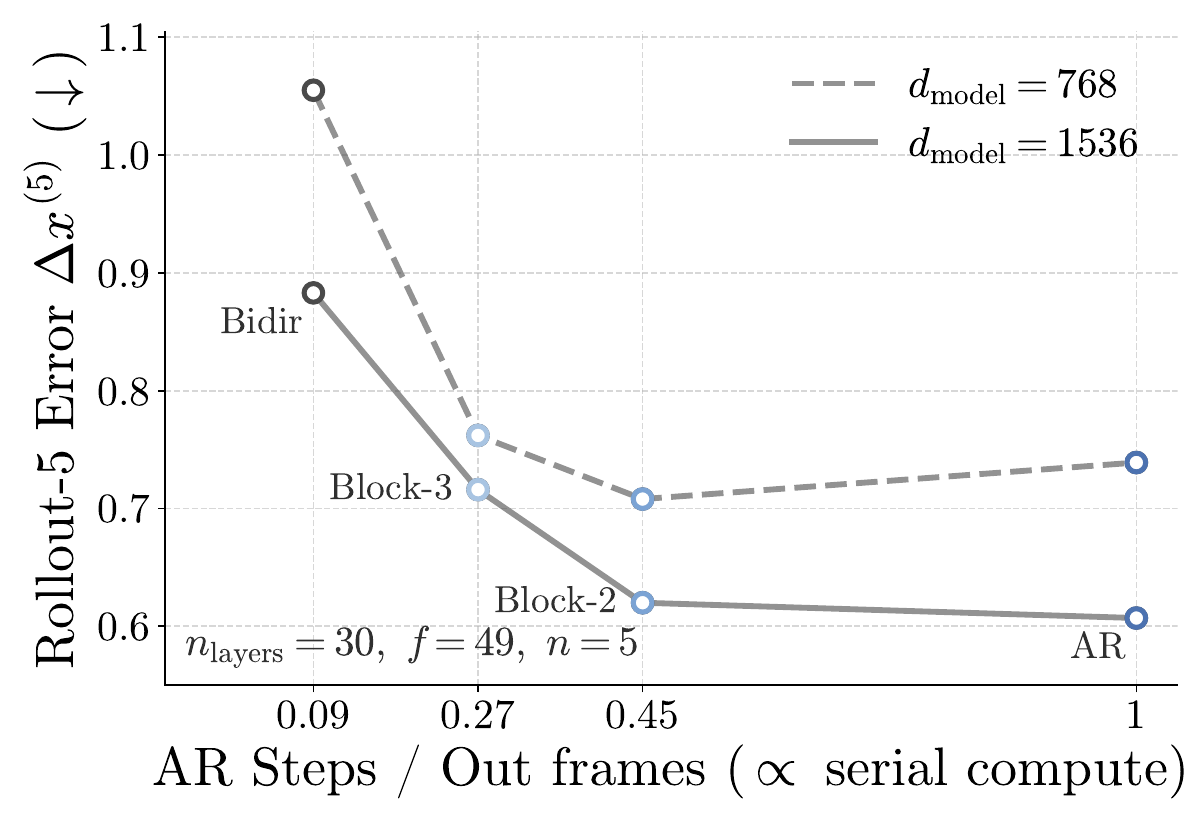}
    \end{subfigure}

    \caption{\textbf{Empirical evidence for the seriality gap.}
    (Top left) Bidirectional diffusion degrades as collision chains lengthen, while autoregressive generation remains more accurate.
    (Top right) Increasing the number of denoising steps does not close the gap.
    (Bottom left) Changing backbone depth and width probes whether architectural serial computation helps.
    (Bottom right) Performance improves as inference becomes more temporally serial, interpolating from bidirectional to blockwise to autoregressive generation.
    All plotted errors are $\Delta x^{(5)}$ (lower is better).}
    \label{fig:main}
\end{figure}

\obshead
\label{obs:one}

\textbf{Bidirectional diffusion degrades as videos get longer.}
In the $n=5$ setting, our filtering criterion makes the minimum number of post-conditioning ball--ball collisions grow with video length $f$, thereby increasing serial complexity (Appendix~\ref{data-synth}).
Although both bidirectional and autoregressive models receive more inference compute at longer horizons, with temporal self-attention scaling roughly as $\mathcal{O}(f^2)$, autoregressive generation maintains lower and nearly flat local physical error $\Delta x^{(5)}$. Bidirectional diffusion instead worsens with $f$, as shown in Figure~\ref{fig:main}~(top left).

\obshead
\label{obs:two}

\textbf{The gap disappears without ball--ball collisions.}
As a length-matched control, we set $n=1$, removing ball--ball collisions and leaving trajectories that can be computed in closed form from the initial state. In this non-serial setting, bidirectional and autoregressive models have nearly identical local physical error $\Delta x^{(5)}$, and both remain stable as $f$ increases, as shown in Figure~\ref{fig:main}~(top left). Together with the $n=5$ result, this control shows that video length alone cannot explain the degradation and supports serial complexity as its key driver.

\obshead
\label{obs:three}

\textbf{Serializing inference improves performance.}
We interpolate from bidirectional to autoregressive diffusion using block-causal masks over temporal latent frames, keeping the backbone fixed. Larger blocks denoise more frames jointly, while smaller blocks impose a more sequential generation order. Local physical error improves as the factorization becomes more serial, following $\text{bidirectional (Block-11)} \rightarrow \text{Block-3} \rightarrow \text{Block-2} \approx \text{autoregressive (Block-1)}$ across model sizes, as shown in Figure~\ref{fig:main}~(bottom right). Since model capacity is unchanged, these gains reflect temporal compute allocation rather than a larger backbone.

\obshead
\label{obs:four}

\textbf{More denoising steps do not close the gap.}
A natural explanation is that bidirectional diffusion may simply be under-computed at inference time.
If denoising iterations supplied the serial computation needed to predict longer collision chains, increasing $T$ should especially improve long-horizon rollouts.
We therefore evaluate bidirectional diffusion with $T \in \{10,20,50,100,200\}$ across all video lengths.
Performance improves up to roughly \(T=50\); beyond this point, additional denoising steps provide no consistent benefit and sometimes increase error, while the long-horizon degradation remains.
Increasing $T$ therefore does not close the gap to autoregressive generation, as shown in Figure~\ref{fig:main}~(top right).

\obshead
\label{obs:five}

\textbf{Depth scaling is more effective than width scaling.}
We next ask whether the missing serial computation can be supplied by the backbone itself.
We sweep bidirectional models over width $d_\text{model}$ and depth $n_\text{layers}$, varying how capacity and compute are allocated.
Bidirectional models benefit more strongly from increasing depth than increasing width up to about $30$ layers. 
At comparable budgets, deeper, narrower models outperform shallower, wider ones, and the gain from $n_\text{layers} = 10$ to $30$ layers exceeds the improvements from widening a $30$-layer model from $d_\text{model}=640$ to $1536$, as shown in Figure~\ref{fig:main}~(bottom left).
Beyond $30$ layers, additional depth gives diminishing returns, possibly because deeper models become harder to optimize \citep{NIPS2015_215a71a1}. 

\begin{figure}[h]
  \centering
  \includegraphics[width=\linewidth]{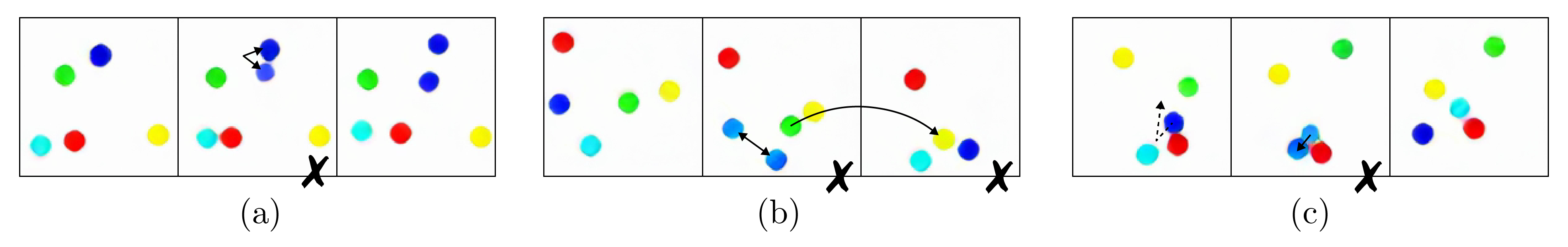}
  \caption{\textbf{Illustrative failure modes in bidirectional generations.}
    We show common qualitative failures produced by the bidirectional diffusion model trained on $f=49$, $n=5$, $d_{\mathrm{model}}=1536$, and $n_{\mathrm{layers}}=30$.
    (a) The dark-blue ball splits into two, resulting in a ball-count violation.
    (b) The dark-blue and light-blue balls exchange colors, while the green and yellow balls merge into one; both failures violate state consistency.
    (c) The dark-blue ball should bounce off the light-blue ball, but instead tunnels through it.
    Additional qualitative examples, including side-by-side comparisons with autoregressive generations, are shown in Appendix~\ref{qualitative}.}
  \label{fig:fails}
  \vspace{-1.0em}
\end{figure}

\paragraph{Qualitative failure modes}
\label{fails}

We observe three recurring state-consistency failures, illustrated in Figure~\ref{fig:fails}: tunneling, ball-count violations, and color-identity violations. In tunneling, a ball passes through another object or disappears and reappears elsewhere without a physically valid intermediate trajectory. In ball-count violations, two balls merge into one or one ball splits into multiple balls. In color-identity violations, a ball changes color, duplicates another ball's color, or swaps its color identity with another ball. These failures can occur in both bidirectional and autoregressive generations, but they are much more frequent in bidirectional models, especially for high-velocity balls and long multi-ball rollouts.

In the backbone depth-width ablation, all model sizes produce visually plausible videos, but smaller and especially shallower backbones exhibit these state-consistency failures more often.

\section{Why Denoising Steps Do Not Provide Scalable Serial Computation}
\label{theory}

The observations above are puzzling. As videos get longer, bidirectional diffusion models already receive roughly $\mathcal{O}(f^2)$ inference compute from self-attention, while the ground-truth simulator requires only $\mathcal{O}(f)$ computation. Moreover, because our main metric is local accuracy, the degradation cannot be explained merely by long-horizon drift. The issue is therefore not simply a lack of total compute, but a failure to convert that compute into scalable serial computation for predicting dependent events.

Through a complexity-theoretic lens, this is not so puzzling. Video diffusion models are asked to solve an inherently serial problem, but their denoising steps do not provide scalable serial computation. To make this precise, we first introduce the relevant complexity classes.

At the \emph{parallelizable} end, the class $\mathsf{TC}^0$ captures polynomial-size, constant-depth threshold circuits; informally, these are computations closer to a closed-form shortcut than to a step-by-step simulation. 

\begin{wrapfigure}{r}{0.48\linewidth}
    \centering
    \includegraphics[width=\linewidth]{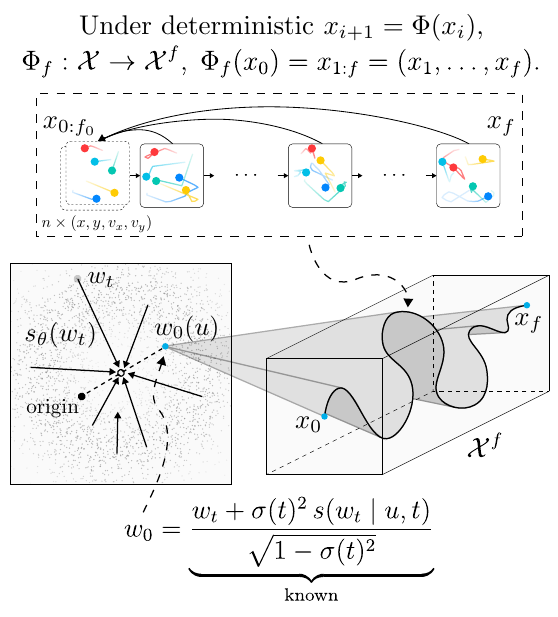}
    \caption{\textbf{Theory overview.}
Deterministic dynamics map an initial state $s_0$ to a unique future trajectory through repeated transitions $s_{i+1}=\Phi(s_i)$.
Given $s_0$, the conditional distribution over futures is therefore a point mass.
Under Gaussian noising, the exact conditional score at any fixed noise level points back to this unique future, so one score-network evaluation can recover it (\cref{eq:one-step-w0-solution}).
Because the clean future is recovered from a single score-network evaluation, denoising steps do not add scalable serial computation beyond what the backbone already computes.}
    \label{fig:theory-overview}
    \vspace{-3.0em}
\end{wrapfigure}

At the \emph{serial} end, $\mathsf{P}$-complete problems capture polynomial-time computations believed not to be efficiently parallelizable.\footnote{Here, ``efficiently parallelizable'' means solvable with polynomial total resources, such as polynomial-size circuits or polynomially many processors. Without this restriction, serial depth can often be reduced by an exponential blow-up in width, which is not considered an efficient parallel algorithm \citep{Greenlaw_1995}. The separation from $\mathsf{P}$-complete problems relies on standard conjectures such as $\mathsf{P}\neq\mathsf{NC}$, which are widely believed but unproved.}

In our case, the multi-ball task shares the elastic-collision primitive with classical $\mathsf{P}$-complete billiard-ball computation~\citep{Fredkin_1982,Greenlaw_1995}. Meanwhile, fixed-depth transformer backbones are bounded-depth parallel computations under standard precision assumptions~\citep{Merrill_2023}. 
Thus, a fixed-depth backbone cannot be the source of serial computation that grows with the task. The natural hope is that repeated denoising steps supply the missing serial computation; we now show why this does not happen for deterministic video prediction.

\paragraph{Deterministic video prediction}
If we assume that the physical process is \emph{deterministic}, then, conditional on the initial state of the system, the set of valid physical trajectories is a \emph{single point} in the space of all possible trajectories, because there is only one correct trajectory starting from that point, no matter how long the trajectory is.

Given a single initial frame of a video, we may not know the exact initial state of the system. For example, we may only observe the positions but not the velocities of bouncing balls. However, in that case, we need to observe just 2 frames to determine both the position and the velocity of each ball, and thus determine the state of the system. In general, if we can determine the exact state of the system after observing $\mathcal{O}(1)$ frames, there is only one valid video trajectory from then on. This condition is known as \emph{full observability}.

Under these assumptions, the video prediction problem reduces to computing a deterministic function:
$$
\mathcal{O}(1) \text{ initial frames} \xrightarrow f \text{ future frames}.
$$
We will formally set up two problems, one about video prediction and one about state prediction. We then argue:

\begin{enumerate}
    \item If the video prediction problem is solvable by a diffusion model with a $\mathsf{TC}^0$ backbone, then its underlying state prediction problem is $\mathsf{TC}^0$.
    \item Bouncing balls and, indeed, most state prediction problems are $\mathsf{P}$-complete and so, under standard assumptions in computational complexity theory, are not in $\mathsf{TC}^0$.
    \item Thus, most video prediction problems are not solvable by a diffusion model with a $\mathsf{TC}^0$ backbone.
\end{enumerate}

From this perspective, the deficiency of a diffusion model is that it converges in $\mathcal{O}(1)$ serial steps, despite expending $\mathcal{O}(f^2)$ parallel compute per serial step. In contrast, an autoregressive model spends $\mathcal{O}(f)$ serial steps, sufficient to solve $\mathsf{P}$-complete problems.

\subsection{Problem Formalism}

Suppose we are given a \emph{deterministic} process. Let the state space be $X$, and let the deterministic time evolution be $\Phi: X \to X$. We have an observation function $F: X \to Y$, which takes a state $x \in X$ and converts it into an image $F(x) \in Y$. Here, $Y$ is the space of possible images. For example, if we consider an image to be an $H \times W$ array of pixels, then $Y = \R^{H \times W}$.

Suppose that the physical process is \emph{fully observable}. That is, we assume that there exists a constant number $f_0$ such that, given any two trajectories
$$
x_0, x_1 = \Phi(x_0), x_2 = \Phi(x_1), \dots; \quad 
x_0', x_1' = \Phi(x_0'), x_2' = \Phi(x_1'), \dots,
$$

if $F(x_0) = F(x_0'), F(x_1) = F(x_1'), \dots F(x_{f_0}) = F(x_{f_0}')$, then $x_0 = x_0'$. 

In other words, we can fully determine $x_0$ if we observe $f_0+1$ consecutive frames $F(x_0), F(x_1), \dots, F(x_{f_0})$.

We also assume that observability is computable in constant time. That is, we have a constant-time algorithm that, given $F(x_{0:f_0})$, returns $x_{0:f_0}$. In our case of bouncing balls, this is clear: given two frames, we can find the velocity of the balls by taking their difference.

The key point is simple. For deterministic video prediction, once the conditioning frames determine the state, there is only one correct future video. After Gaussian noising, the conditional distribution remains centered on that unique future, so the exact score at any fixed noise level contains enough information to recover it in one step. Repeated denoising therefore does not add scalable serial computation beyond what the backbone already provides, as summarized in Figure~\ref{fig:theory-overview}.

Given this setup, we can define two formal problems:
\begin{enumerate}
    \item \textbf{Video segment prediction}: given an initial video segment $F(x_{0:f_0})$ of a physical trajectory and a number $f$, compute its future video segment $F(x_{f_0+1:f})$.
    \item \textbf{State prediction}: given an initial physical trajectory $x_{0:f_0}$ and a number $f$, compute $x_f$.
\end{enumerate}

Suppose further that we have a diffusion model that solves the following problem: given an initial video segment $y_{0:f_0} := F(x_{0:f_0})$ of a physical trajectory and a number $f$, compute its future video trajectory $y_{f_0+1:f} := F(x_{f_0+1:f})$.

\subsection{Diffusion Modeling}

Now, we specify what diffusion modeling means. We use the noise-conditional score-matching network (NCSN) formalism. Diffusion models in other formalisms can be converted to an equivalent NCSN diffusion model \citep{luo2022understandingdiffusionmodelsunified}. The full formalism is in Appendix~\ref{app:theory}.

Diffusion modeling is a technique to sample from a conditional probability distribution. Let $w$ be an object we want to sample, and let $u$ be a conditioning input. $w$ is an element of a sample space, here taken to be $\R^d$, where $d$ is some positive integer.\footnote{We picked the letters $w, u$ only to avoid confusion, since we have already assigned the letters $x, y$.} For example, in our case of video modeling, the condition is $y_{0:f_0}$, the initial video segment, and the object we want to sample is $y_{f_0+1:f}$, the rest of the video, within the sample space $\R^{H \times W \times (f-f_0)}$.

Now, we define a forward diffusion process. We use the standard Gaussian diffusion process. The process is defined by a noise rate function $\beta: [0, \infty) \to (0, \infty)$, with the condition that $\int_0^\infty \beta(t) dt = \infty$. We first sample a starting point $w_0 \sim \rho_0(w|u)$ and then incrementally add white noise at a rate of $\beta(t)$. The accumulated effect is that, at $t$, the noised sample has the following distribution:
\begin{equation}
w_t = \sqrt{1-\sigma(t)^2} w_0 + \sigma(t) z, \quad z \sim \mathcal N(0, I_d),
\end{equation}

where $\sigma(t)^2 = 1- e^{-\int_0^t \beta(\tau) d\tau}$ is the variance due to accumulated noise.

In other words, the probability distribution of $w_t$, conditional on $u$, is the convolution of a scaled version of $\rho_0(\cdot | u)$ with a Gaussian function. The probability distributions $\rho(\cdot | u, t)$ are increasingly blurry versions of the scaled original $\rho_0(\cdot | u)$, starting at $\rho(\cdot | u, 0) = \rho_0(\cdot | u)$ and ending with $\lim_{t\to\infty} \rho(\cdot | u, \infty) = \rho_{\mathcal N(0, I_d)}$, where we lose all traces of the original and obtain the standard Gaussian distribution.

Define the score function 
\begin{equation}
    s(w_t | u, t) := \nabla_{w_t} \log \rho(w_t | u, t).
    \label{eq:score_fn}
\end{equation}

The score function can be used to define a backward stochastic process, starting from $\mathcal N(0, I_d)$ and ending with the original $\rho_0(\cdot|u)$. In the NCSN formalism, we use a neural network $s_\theta$ to approximate $s$. This network is the titular ``noise conditional score network'' (NCSN). We then approximate the backward process by numerical stochastic integration.

The exact score function $s$ has no closed-form solution in general, because $\rho_0(w_0 | u)$ is too generic to be integrated. However, if $w_0$ is a \textit{deterministic} function of $u$, then the score function has a closed-form solution, which allows us to solve for $w_0(u)$ in one step:
\begin{equation}\label{eq:one-step-w0-solution}
    w_0(u) = \frac{\sigma(t)^2 s(w_t | u, t) + w_t}{\sqrt{1-\sigma(t)^2}}.
\end{equation}

In particular, this means that if we have a neural network that computes the score function, we can use it to compute $w_0(u)$ for \textit{any} choice of $t > 0$ and \textit{any} choice of $w_t$. Intuitively, this is possible because the flows in the corresponding flow-matching problem are just straight lines and can therefore be solved in one step without evaluating any integrals (Figure~\ref{fig:theory-overview}).

Now we can frame a video diffusion model in this formalism. Its backbone network takes as inputs an initial video segment $y_{0:f_0}$, a positive integer $f > f_0$, a noise level $t$, and a noisy future video segment $\hat y_{f_0+1:f}$. It outputs the noise-conditional score:
\begin{equation}
    s(\hat y_{f_0+1: f} | y_{0:f_0}, f, t) = \nabla_{\hat y_{f_0+1: f}} \log \rho(\hat y_{f_0+1: f} | y_{0:f_0}, t).
\end{equation}

\begin{proposition}\label{prop1}
If a diffusion model solves the video prediction problem and its NCSN $s_\theta$ is a $\mathsf{TC}^0$ network computing the exact score~\eqref{eq:score_fn}, then the state prediction problem is in $\mathsf{TC}^0$.
\end{proposition}

\begin{proof}
We solve the state prediction problem by an algorithm that calls the NCSN just once. The entire algorithm runs in $\mathsf{TC}^0$. The details are in Appendix~\ref{app:theory}.
\end{proof}

\textit{Comment.} Is it unrealistic to assume that the score function is matched exactly? We use exact score matching to present the core argument as clearly as possible, but the conclusion does not require exact equality. Even when $s_\theta \approx s$, we can still prove that the state prediction problem is in $\mathsf{TC}^0$ given a sufficiently strong upper bound on $|s_\theta-s|$; see~\cite[App.~F]{liu2026serialscalinghypothesis}.

The diffusion models considered here are trained to approximate the score, so large score errors reflect optimization failure or insufficient model capacity. Scaling model size and training compute should improve score estimation. This makes our proposition more relevant at scale.

On a positive note, our result points toward a different class of iterative generative models. Rather than training each iteration only to approximate the score, such models could train their intermediate states to perform and preserve useful computation across iterations. Developing objectives that enable this behavior is an open direction.

\section{Discussion and Limitations}

\paragraph{The testbed isolates dependent-event prediction}
Hard-sphere dynamics is a diagnostic abstraction: by removing visual realism, semantic ambiguity, stochasticity, and open-world uncertainty, it isolates a single question: can a video diffusion model track a growing chain of dependent events?
The seriality gap is therefore not about bouncing balls or physics specifically, but about a broader class of tasks whose predictions require serial chains of deterministic transitions through intermediate states.

\paragraph{Autoregression is not a universal fix}
Autoregression helps in our setting because the dependencies are mostly temporal: each short future block depends on the recent past, so serial generation matches the causal order.
Other visual reasoning tasks need not factor this way.
Visual-reasoning studies instead find early global planning or progressive refinement across denoising steps \citep{wiedemer2025videomodelszeroshotlearners,newman2026videomodelsreasonearly,wang2026demystifying}; such tasks need not factor into frame-aligned transitions.
Autoregression helps when its factorization aligns with task dependency structure; otherwise, the needed serial computation must come from depth, search, planning, or another mechanism.

Moreover, even if more serial computation is supplied, it does not guarantee that a model will use it to track event dependencies. In natural video, finite model capacity must also represent appearance and stochasticity. These competing demands may dominate learning, allowing models--autoregressive or bidirectional--to favor visually plausible outputs over faithful dependent-event prediction. Our controlled testbed removes these confounds and therefore establishes a clear computational advantage, not a guarantee of success on unconstrained real-world video.

\paragraph{Limitations} 
Our theory applies to deterministic video prediction and is one-way: it does not characterize every setting in which denoising steps fail to provide scalable serial computation.
Stochasticity or noise can break the assumptions, but this does not imply diffusion models acquire scalable serial computation.
Indeed, our vibrating-wall hard-sphere dynamics variant in Appendix~\ref{app:noisy-hard-sphere} shows the same qualitative pattern: even when the deterministic assumption is violated, autoregressive generation remains more locally consistent than bidirectional denoising.
This suggests that violating the theorem's assumptions does not necessarily remove the seriality gap.

\section{Related Work}
\label{sec:related_work}

\paragraph{Video models as simulators and reasoners}
Video generators are increasingly framed as more than synthesis systems: they are proposed as world simulators, physical reasoners, and visual planning models \citep{videoworldsimulators2024,wan2025wanopenadvancedlargescale,wiedemer2025videomodelszeroshotlearners,wang2026vbvr,NEURIPS2025_f0552f14}. This has led to benchmarks that test physical consistency, temporal coherence, and physical commonsense, including PhyGenBench, T2VPhysBench, Physics-IQ, Morpheus, VBench, and VBench2 \citep{meng2024worldsimulatorcraftingphysical,guo2025t2vphysbenchfirstprinciplesbenchmarkphysical,motamed2025generativevideomodelsunderstand,raedsch2026physicsiqverified,tragoudaras2026evaluatingnewtonianmechanicsvideo,Huang_2024_CVPR,zheng2025vbench20advancingvideogeneration}. Recent maze-solving studies push the same view toward visual planning and reasoning \citep{wiedemer2025videomodelszeroshotlearners,newman2026videomodelsreasonearly}. Synthetic physics environments provide a complementary way to remove visual realism as a confound: \citet{pmlr-v267-kang25g} study video diffusion models in the PHYRE physical simulator \citep{NEURIPS2019_4191ef5f}, showing that models can struggle with physical laws even in visually simple settings. Our hard-sphere dataset follows this controlled-evaluation philosophy, but focuses on the question: can video diffusion models track a chain of dependent events when the ground truth is unambiguous and visual realism is not a confound?

\paragraph{Video diffusion architectures and inference paradigms}
Modern video generators commonly apply latent diffusion to spacetime tensors using transformer or U-Net backbones with temporal attention \citep{NEURIPS2022_39235c56,blattmann2023stablevideodiffusionscaling,Bar_Tal_2024,wan2025wanopenadvancedlargescale}. The dominant inference paradigm is bidirectional denoising: all frames are refined jointly by a bounded-depth backbone with access to both past and future frames. Recent work explores autoregressive and sliding-window diffusion for streaming and long-horizon generation \citep{Weng_2024,Xie_2025_CVPR,Yin_2025_CVPR}. We treat these inference choices not only as engineering tradeoffs, but as different ways of allocating serial computation across time.

\paragraph{Other temporal bottlenecks in video diffusion}
Prior work targets temporal bottlenecks through noise rescheduling \citep{Gupta_2024,qiu2024freenoisetuningfreelongervideo}, joint appearance--motion representations \citep{pmlr-v267-chefer25a}, temporal regularization \citep{chen2025temporalregularizationmakesvideo}, or preservation of few-step motion priors during refinement \citep{han2026phaselock}. These methods improve temporal structure; our focus is different: better priors may improve coherence but do not address bidirectional denoising's lack of scalable serial computation for long dependency chains.

Our explanation connects these empirical failures to prior work on transformer parallelism, depth--width tradeoffs, and serial scaling, which argues that model depth, width, and denoising steps are not interchangeable computational resources \citep{Merrill_2023,pmlr-v49-telgarsky16,chen2024theoreticallimitationsmultilayertransformer,liu2026serialscalinghypothesis}.

\begin{ack}

We thank Alexei Efros for his constant efforts to debate with us.
We thank Nicolas Dufour for suggestions for improving experimental rigor.
We thank Amil Dravid for his suggestions to add ``more dots'' to the experiments section.
We thank other members of Berkeley AI Research for helpful discussions.
KP and YB are supported by ONR MURI.

\end{ack}

{\small
\bibliographystyle{unsrtnat}
\bibliography{main}
}

\appendix

\newpage
\section{Theoretical Details}\label{app:theory}

As before, let $w$ be an object we want to sample, and let $u$ be a conditioning input. $w$ is an element of a sample space, here taken to be $\R^d$, where $d$ is some positive integer.

The forward diffusion process is defined by a noise rate function $\beta: [0, \infty) \to (0, \infty)$, with the condition that $\int_0^\infty \beta(t) dt = \infty$. We begin by sampling a starting point $w_0 \sim \rho_0(w|u)$ and then define the forward stochastic process:
$$
w_{t + dt} = w_t - \underbrace{\frac 12 \beta(t) w_t dt}_{\text{drift}} + \underbrace{\sqrt{\beta(t)} dW_t}_{\text{noise}},
$$
for all $t \in [0, \infty)$, where $W_t$ is a standard Brownian motion in $\R^d$. The $W$ stands for ``Wiener'', since this is also called the Wiener process. It can be schematically thought of as $dW_t \sim \mathcal N(0, dt I_d)$.

This is a stochastic differential equation with a closed-form solution:
$$
w_t = \sqrt{1-\sigma(t)^2} w_0 + \sigma(t) z, \quad z \sim \mathcal N(0, I_d),
$$
where $\sigma(t)^2 = 1- e^{-\int_0^t \beta(\tau) d\tau}$ is the variance due to accumulated noise.

In other words, the probability distribution of $w_t$, conditional on $u$, is the convolution of a scaled version of $\rho_0(\cdot | u)$ with a Gaussian function. The probability distributions $\rho(\cdot | u, t)$ are increasingly blurry versions of the original $\rho_0(\cdot | u)$, starting at $\rho(\cdot | u, 0) = \rho_0(\cdot | u)$ and ending with $\lim_{t\to\infty} \rho(\cdot | u, \infty) = \rho_{\mathcal N(0, I_d)}$, where we lose all traces of the original and obtain the standard Gaussian distribution.

Define the score function 
$$
s(w_t | u, t) := \nabla_{w_t} \log \rho(w_t | u, t).
$$

With the score function, the same process can be expressed in the time-reversed form \citep{Anderson_1982}:
$$w_{t-dt} = w_t + \frac 12 \beta(t) w_t dt + \beta(t) \underbrace{s(w_t | u, t)}_{\text{score function}} dt + \sqrt{\beta(t)} dW_t.$$

In the NCSN formalism, we have a neural network $s_\theta \approx s$. We then solve the backward process by numerical stochastic integration:
$$\hat w_{t-\Delta t} = \hat w_t + \frac 12 \beta(t)\hat w_t \Delta t + \beta(t) s_\theta(\hat w_t | u, t)\Delta t + \sqrt{\beta(t)} \Delta W_t.$$

The exact score function $s$ has no closed-form solution in general, because $\rho_0(w_0 | u)$ is too generic to be integrated. However, if $w_0$ is a \textit{deterministic} function of $u$, then
$$
\begin{aligned}
\log\rho(w_t | u, t) &= -\frac{1}{2\sigma(t)^2}\left\|w_t - \sqrt{1- \sigma(t)^2}w(u)\right\|^2 + \mathrm{Const}
\\
s(w_t | u, t) &= \frac{\sqrt{1-\sigma(t)^2} w_0(u) - w_t}{\sigma(t)^2}.
\end{aligned}$$

Thus, we are actually able to use the score function to solve for $w_0(u)$ in one step:
$$
w_0(u) = \frac{\sigma(t)^2 s(w_t | u, t) + w_t}{\sqrt{1-\sigma(t)^2}}.
$$

\textit{Proof of \cref{prop1}.}
\begin{enumerate}
    \item Assign $t$ to be any constant. Any constant will be fine. For example, $t := 1$ will do.
    \item Assign $\hat y_{f_0+1: f, t}$ to be any constant. Any constant will be fine. For example, this will do:
    $$\hat y_{f_0+1: f, t} = (0, 0, \dots, 0),$$
    where $0$ stands for a blank image.
    \item Compute the score
    $$s_t := s(\hat y_{f_0+1: f} | y_{0:f_0}, f, t)$$
    using the NCSN backbone.
    \item Solve for the future video segment
    $$
    y_{f_0+1: f} := \frac{\sigma(t)^2 s_t + \hat y_{f_0 + 1:f, t}}{\sqrt{1-\sigma(t)^2}}.
    $$
    \item Invert the last segment of the video trajectory to its state:
    $$
    x_{f-f_0:f} = F^{-1}(y_{f-f_0:f}).
    $$
    \item Output $x_f$.
\end{enumerate}
By assumption, the NCSN backbone is a $\mathsf{TC}^0$ network, and every other step also runs in $\mathcal{O}(1)$ serial time. Thus, the algorithm is in $\mathsf{TC}^0$, and therefore the state prediction problem is also in $\mathsf{TC}^0$. \qed

\section{Numerical Results}
\label{app:numerical-results}

\begin{table}[h]
\centering
\caption{\textbf{Results under different video lengths and inference factorizations.}
Physical accuracy and visual fidelity for bidirectional, blockwise, and autoregressive diffusion models on $n=5$ hard-sphere dynamics videos spanning a range of lengths $f$.
Lower $\Delta x^{(\mathrm{GT})}$, $\Delta x^{(1)}$, and $\Delta x^{(5)}$ indicate higher physical accuracy; higher IoU indicates higher visual fidelity.
The results support \cref{obs:one,obs:three}: bidirectional models degrade with longer chains of dependent events, while temporal serialization improves local physical accuracy.}
\begin{tabular}{ccccccccc}
\toprule
Type & $d_{\text{model}}$ & $n_{\text{layers}}$ & $n_\text{params}$ & TFLOPs & $\Delta x^{(\text{GT})}$ $\downarrow$ & $\Delta x^{(1)}$ $\downarrow$ & $\Delta x^{(5)}$ $\downarrow$ & IoU $\uparrow$ \\

\midrule

\rowcolor{gray!10}
\multicolumn{9}{l}{\textbf{$f = 25$ frames, $n = 5$ balls}} \\
\addlinespace[2pt]
Bidir & 768  & 30 & 217M & 10.4 & 1.441 & 0.197 & 0.799 & 0.703 \\
Bidir & 1536 & 30 & 867M & 39.9 & 1.265 & 0.160 & 0.670 & 0.726 \\
AR    & 768  & 30 & 217M & 7.3  & 1.179 & 0.166 & 0.664 & 0.714 \\
AR & 1536 & 30 & 867M & 28.1 & 1.042 & 0.148 & 0.613 & 0.730 \\

\midrule

\rowcolor{gray!10}
\multicolumn{9}{l}{\textbf{$f = 33$ frames, $n = 5$ balls}} \\
\addlinespace[2pt]
Bidir & 768  & 30 & 217M & 13.8 & 2.047 & 0.210 & 0.873 & 0.693 \\
Bidir & 1536 & 30 & 867M & 52.0 & 1.944 & 0.164 & 0.718 & 0.725 \\
AR    & 768  & 30 & 217M & 10.3 & 1.905 & 0.175 & 0.712 & 0.705 \\
AR    & 1536 & 30 & 867M & 39.6 & 1.645 & 0.142 & 0.604 & 0.735 \\

\midrule

\rowcolor{gray!10}
\multicolumn{9}{l}{\textbf{$f = 41$ frames, $n = 5$ balls}} \\
\addlinespace[2pt]
Bidir & 768  & 30 & 217M & 17.2 & 2.581 & 0.227 & 0.952 & 0.696 \\
Bidir & 1536 & 30 & 867M & 64.4 & 2.402 & 0.200 & 0.856 & 0.698 \\
AR    & 768  & 30 & 217M & 13.4 & 2.423 & 0.167 & 0.695 & 0.712 \\
AR    & 1536 & 30 & 867M & 51.3 & 2.165 & 0.142 & 0.606 & 0.732 \\

\midrule

\rowcolor{gray!10}
\multicolumn{9}{l}{\textbf{$f = 49$ frames, $n = 5$ balls}} \\
\addlinespace[2pt]
Bidir & 768 & 30 & 217M & 20.9 & 2.923 & 0.261 & 1.055 & 0.691 \\
Bidir & 1536 & 30 & 867M & 77.0 & 2.868 & 0.205 & 0.883 & 0.704 \\
Block-3 & 768 & 30 & 217M & 16.9 & 2.749 & 0.182 & 0.762 & 0.713 \\
Block-3 & 1536 & 30 & 867M & 63.7 & 2.675 & 0.174 & 0.716 & 0.705 \\
Block-2 & 768 & 30 & 217M & 16.7 & 2.717 & 0.169 & 0.708 & 0.715 \\
Block-2 & 1536 & 30 & 867M & 63.4 & 2.578 & 0.145 & 0.620 & 0.731 \\
AR & 768 & 30 & 217M & 16.6 & 2.798 & 0.178 & 0.739 & 0.703 \\
AR & 1536 & 30 & 867M & 63.1 & 2.548 & 0.142 & 0.607 & 0.733 \\

\bottomrule
\end{tabular}
\end{table}

\begin{table}[h]
\centering
\caption{\textbf{Single-ball control as a non-serial setting.}
Physical accuracy and visual fidelity for bidirectional and autoregressive diffusion models on $n=1$ hard-sphere dynamics videos of different lengths $f$.
Unlike the multi-ball setting, longer single-ball videos do not introduce longer chains of dependent events.
The results support \cref{obs:two}: the bidirectional--autoregressive gap largely disappears without serial dependence.}
\begin{tabular}{ccccccccc}
\toprule
Type & $d_{\text{model}}$ & $n_{\text{layers}}$ & $n_{\text{params}}$ & TFLOPs
& \makecell{$\Delta x^{(\mathrm{GT})} \downarrow$\\$n=1$}
& \makecell{$\Delta x^{(5)} \downarrow$\\$n=1$}
& \makecell{$\Delta x^{(5)} \downarrow$\\$n=5$}
& \makecell{IoU $\uparrow$\\$n=1$} \\
\midrule

\rowcolor{gray!10}
\multicolumn{9}{l}{\textbf{$f = 25$ frames}} \\
\addlinespace[2pt]
Bidir & 768  & 30 & 217M & 10.4 & 0.310 & 0.220 & 0.799 & 0.823 \\
AR    & 768  & 30 & 217M & 7.3  & 0.162 & 0.201 & 0.664 & 0.819 \\

\midrule

\rowcolor{gray!10}
\multicolumn{9}{l}{\textbf{$f = 49$ frames}} \\
\addlinespace[2pt]
Bidir & 768  & 30 & 217M & 20.9 & 1.417 & 0.246 & 1.055 & 0.810 \\
AR    & 768  & 30 & 217M & 16.6 & 0.471 & 0.220 & 0.739 & 0.815 \\

\bottomrule
\end{tabular}
\end{table}

\begin{table}[h]
\centering
\caption{\textbf{More denoising steps do not close the seriality gap.}
Local physical accuracy for bidirectional diffusion models with $T \in \{10,20,50,100,200\}$ denoising steps on $n=5$ hard-sphere dynamics videos of different lengths $f$.
Performance plateaus after a moderate number of steps, while longer videos remain more difficult.
The results support \cref{obs:four}: denoising depth does not provide scalable serial computation for longer chains of dependent events.}
\begin{tabular}{ccccccccc}
\toprule
Type & $d_{\text{model}}$ & $n_{\text{layers}}$ & TFLOPs
& \makecell{$\Delta x^{(5)} \downarrow$\\$T=10$}
& \makecell{$\Delta x^{(5)} \downarrow$\\$T=20$}
& \makecell{$\Delta x^{(5)} \downarrow$\\$T=50$}
& \makecell{$\Delta x^{(5)} \downarrow$\\$T=100$}
& \makecell{$\Delta x^{(5)} \downarrow$\\$T=200$} \\
\midrule

\rowcolor{gray!10}
\multicolumn{9}{l}{\textbf{$f = 25$ frames, $n=5$ balls}} \\
\addlinespace[2pt]
Bidir & 1536 & 30 & 39.9 & 0.690 & 0.667 & 0.670 & 0.679 & 0.689 \\

\midrule

\rowcolor{gray!10}
\multicolumn{9}{l}{\textbf{$f = 33$ frames, $n=5$ balls}} \\
\addlinespace[2pt]
Bidir & 1536 & 30 & 52.0 & 0.746 & 0.719 & 0.718 & 0.724 & 0.740 \\

\midrule

\rowcolor{gray!10}
\multicolumn{9}{l}{\textbf{$f = 41$ frames, $n=5$ balls}} \\
\addlinespace[2pt]
Bidir & 1536 & 30 & 64.4 & 0.946 & 0.886 & 0.856 & 0.847 & 0.848 \\

\midrule

\rowcolor{gray!10}
\multicolumn{9}{l}{\textbf{$f = 49$ frames, $n=5$ balls}} \\
\addlinespace[2pt]
Bidir & 1536 & 30 & 77.0 & 0.991 & 0.931 & 0.883 & 0.889 & 0.916 \\

\bottomrule
\end{tabular}
\end{table}

\begin{table}[h]
\centering
\caption{\textbf{Depth scaling is more effective than width scaling.}
Physical accuracy and visual fidelity for bidirectional diffusion models with different widths $d_{\mathrm{model}}$ and depths $n_{\mathrm{layers}}$ on $n=5$ hard-sphere dynamics videos.
Depth improves physical accuracy more strongly than width up to about $30$ layers, beyond which further improvements are unclear.
The results support \cref{obs:five}: useful serial computation is supplied more effectively by backbone depth than by width alone.}
\begin{tabular}{ccccccccc}
\toprule
Type & $d_{\text{model}}$ & $n_{\text{layers}}$ & $n_\text{params}$ & TFLOPs & $\Delta x^{(\text{GT})}$ $\downarrow$ & $\Delta x^{(1)}$ $\downarrow$ & $\Delta x^{(5)}$ $\downarrow$ & IoU $\uparrow$ \\

\midrule

\rowcolor{gray!10}
\multicolumn{9}{l}{\textbf{$f = 49$ frames, $n = 5$ balls}} \\
\addlinespace[2pt]
Bidir & 640  & 30 & 151M & 14.9 & 2.973 & 0.307 & 1.172 & 0.665 \\
Bidir & 768  & 20 & 146M & 14.0 & 3.038 & 0.403 & 1.420 & 0.640 \\
Bidir & 768  & 30 & 217M & 20.9 & 2.923 & 0.261 & 1.055 & 0.691 \\
Bidir & 768  & 60 & 430M & 41.7 & 2.839 & 0.233 & 0.950 & 0.696 \\
Bidir & 768  & 90 & 643M & 62.6 & 2.788 & 0.210 & 0.881 & 0.700 \\
Bidir & 1152 & 10 & 169M & 14.8 & 3.076 & 0.416 & 1.466 & 0.653 \\
Bidir & 1152 & 20 & 329M & 29.7 & 2.941 & 0.261 & 1.060 & 0.690 \\
Bidir & 1152 & 30 & 488M & 44.5 & 2.822 & 0.227 & 0.951 & 0.694 \\
Bidir & 1152 & 50 & 807M & 74.2 & 2.739 & 0.251 & 1.000 & 0.676 \\
Bidir & 1536 & 10 & 301M & 25.7 & 2.939 & 0.324 & 1.230 & 0.680 \\
Bidir & 1536 & 30 & 867M & 77.0 & 2.868 & 0.205 & 0.883 & 0.704 \\

\bottomrule
\end{tabular}
\end{table}

\FloatBarrier

\section{Qualitative Results}
\label{qualitative}

\begin{figure}[h]
  \centering

  \begin{subfigure}{\linewidth}
    \centering

    \includegraphics[width=\linewidth]{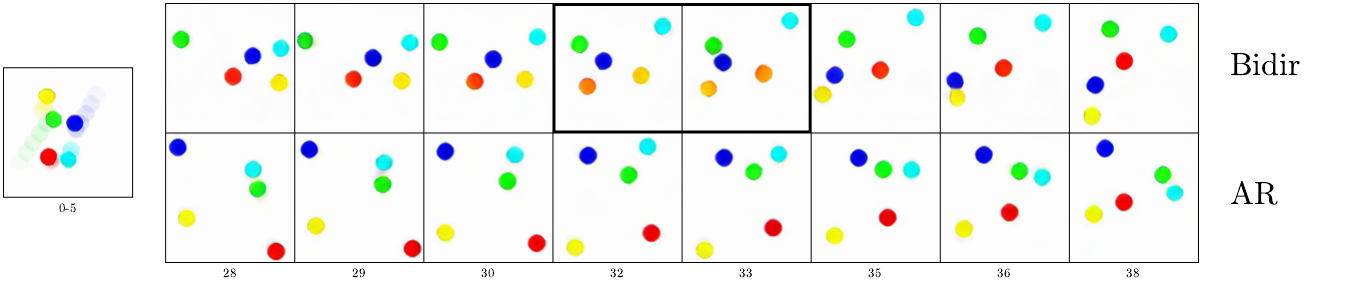}

    \vspace{0.5em}

    \includegraphics[width=\linewidth]{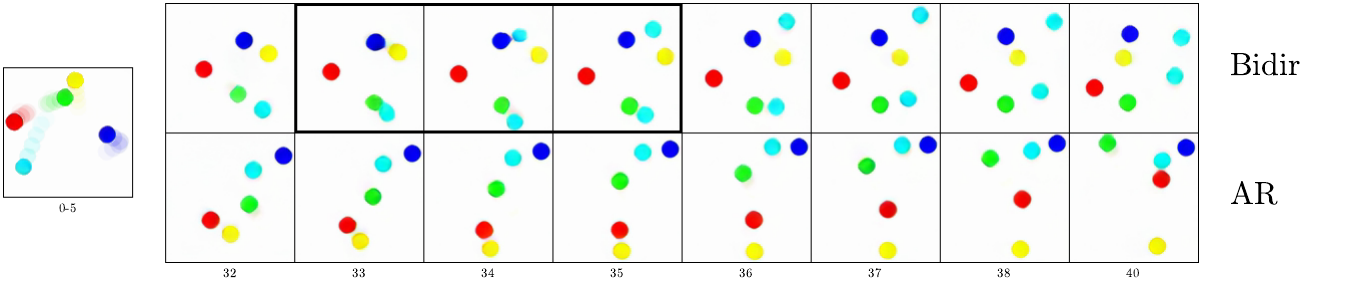}

    \vspace{0.5em}

    \includegraphics[width=\linewidth]{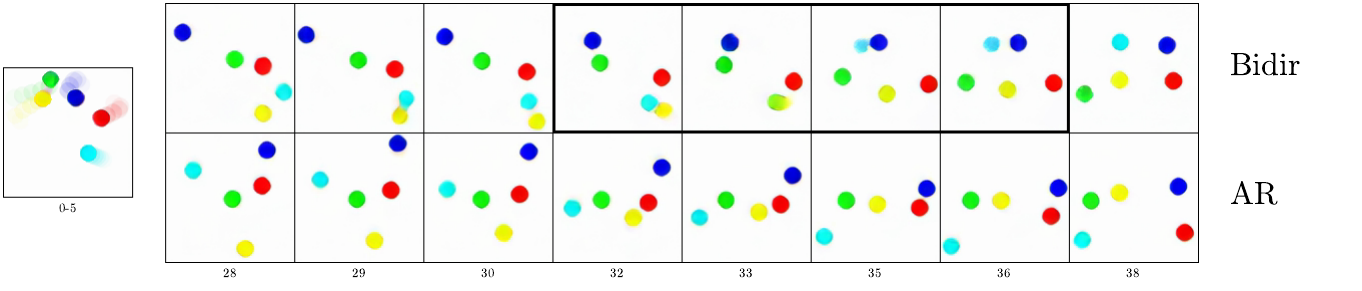}

    \vspace{0.5em}

    \includegraphics[width=\linewidth]{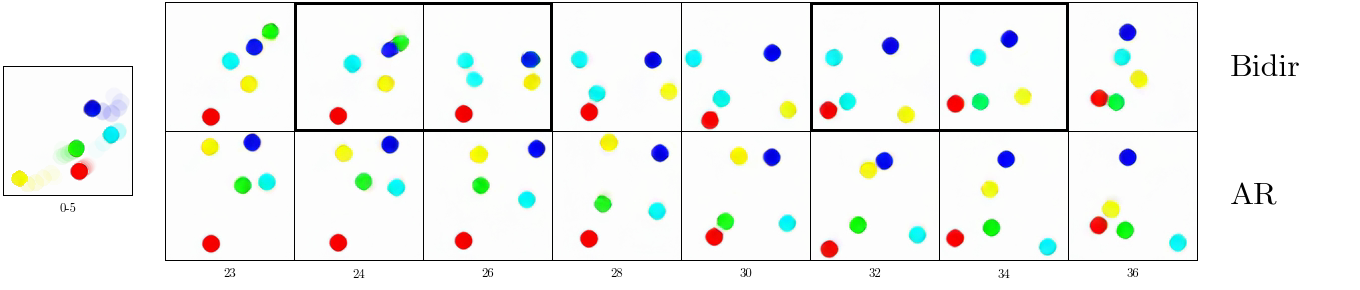}

    \vspace{0.5em}

	    \includegraphics[width=\linewidth]{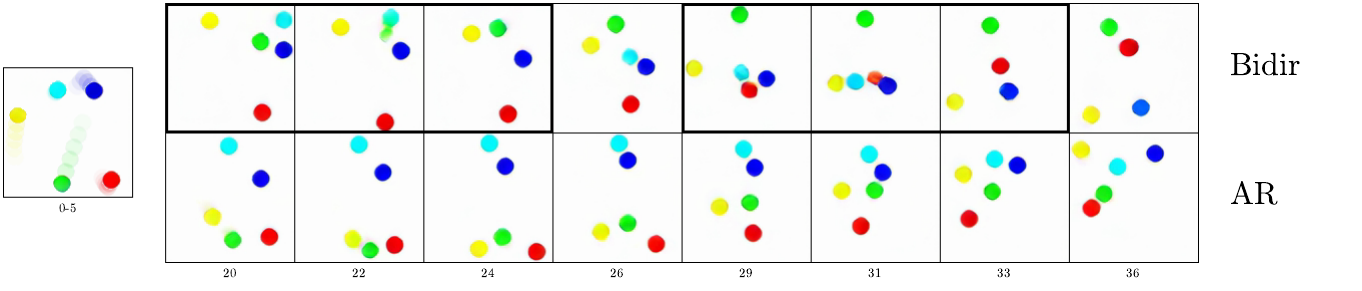}

    \caption{}
    \label{fig:qualitative-a}
  \end{subfigure}

  \caption{\textbf{Qualitative failures.}
  We compare bidirectional and autoregressive outputs from models with $d_{\mathrm{model}}=1536$ and $n_{\mathrm{layers}}=30$.
  The left column shows the shared conditioning frames with motion trails; the remaining columns show $8$ output frames, with frame indices below.
  Thick borders highlight frames where failures appear.
  Bidirectional outputs exhibit state-consistency failures such as tunneling, ball-count violations, and color-identity violations, while the autoregressive counterparts remain more physically consistent.}
  \label{fig:qualitative}
\end{figure}

\clearpage

\begin{figure}[h]\ContinuedFloat
  \centering

  \begin{subfigure}{\linewidth}
    \centering

    \includegraphics[width=\linewidth]{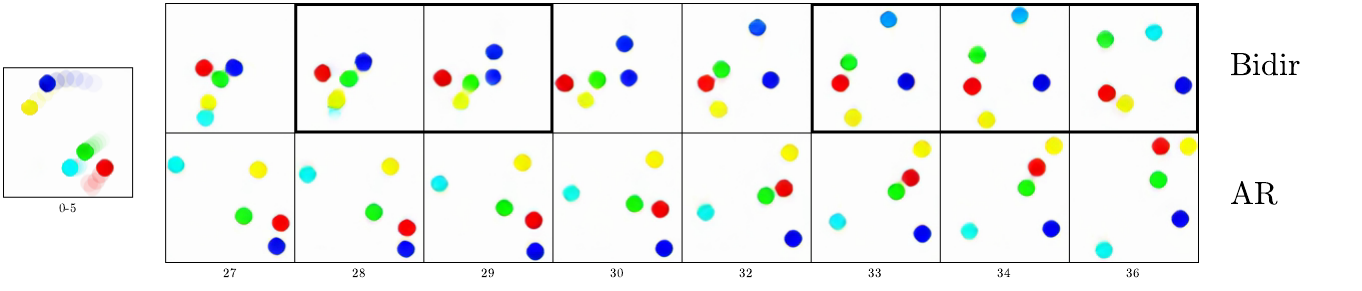}

    \vspace{0.5em}

    \includegraphics[width=\linewidth]{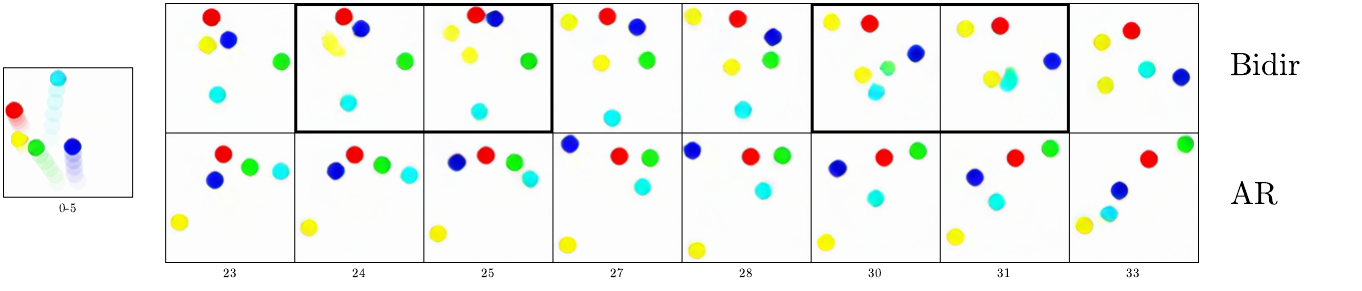}

    \vspace{0.5em}

    \includegraphics[width=\linewidth]{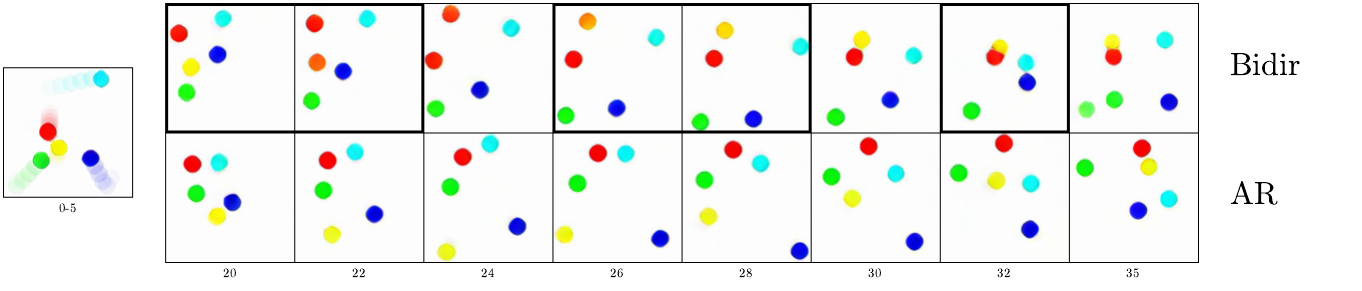}

    \caption{}
    \label{fig:qualitative-b}
  \end{subfigure}

	  \caption{\textbf{Qualitative failures, continued.}
	  Illustrations follow the same format as Figure~\ref{fig:qualitative-a}.}
\end{figure}

\FloatBarrier

\section{Vibrating-Wall Hard-Sphere Dynamics}
\label{app:noisy-hard-sphere}

In Section~\ref{theory}, we formulate our theory for deterministic video prediction: conditioned on a fully observed initial state, the future trajectory is unique.
Natural videos, however, often involve stochasticity, partial observability, or unmodeled perturbations.
To test whether the seriality gap persists beyond the deterministic setting, we conduct a preliminary experiment on a stochastic variant of the hard-sphere dynamics dataset.

In this vibrating-wall variant, wall collisions are no longer perfectly deterministic. Whenever a ball bounces off a wall, we perturb its speed by a random amount
\[
s_{\mathrm{after}} = |s_{\mathrm{before}} + \delta s|,
\qquad\delta s \sim \mathrm{Unif}[-6,6],
\]
while keeping the direction of the post-collision velocity unchanged. Thus, even with the same initial state, the future trajectory is no longer unique: different wall vibrations induce different valid rollouts.

The vibrating-wall setting breaks the point-mass assumption used in the deterministic theory, but it does not eliminate all serial structure. Between vibrating-wall collisions, the motion remains deterministic and follows the same local physical laws as in the original hard-sphere system. We therefore evaluate local physical consistency only on deterministic segments. Concretely, we compute the rollout-5 error $\Delta x^{(5)}$ after excluding any frame for which the five-step rollout window contains a ball-to-wall collision. This isolates whether the generated trajectory is locally consistent within intervals where the dynamics are deterministic.

\begin{figure}[h]
  \centering
  \includegraphics[width=0.5\linewidth]{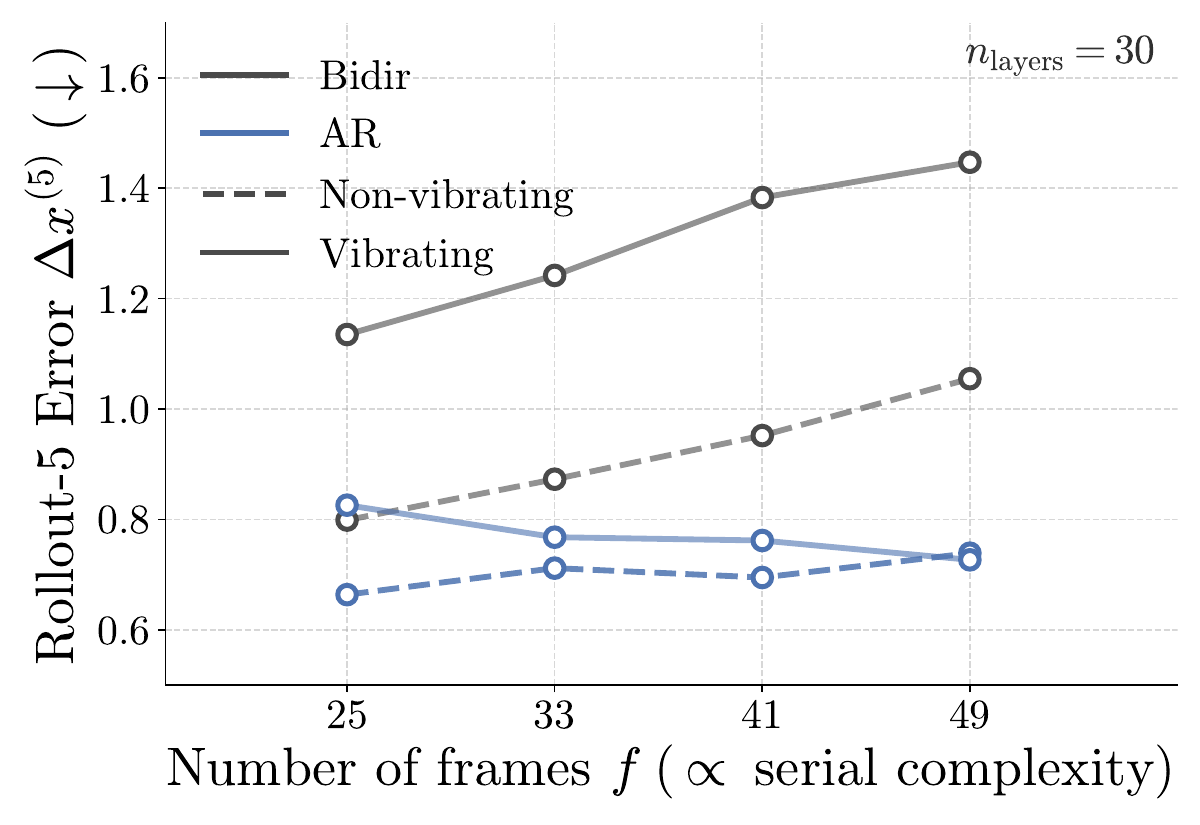}
  \caption{\textbf{Preliminary results on vibrating-wall hard-sphere dynamics.}
  We introduce stochasticity by perturbing the speed of a ball after each wall collision, while leaving the velocity direction unchanged.
  Evaluation reports rollout-5 error only on deterministic segments, excluding windows that contain a wall collision.
  Although preliminary, the results show the same qualitative pattern as the deterministic setting, suggesting that the seriality gap is not merely an artifact of globally deterministic rollouts.}
  \label{fig:noisy}
\end{figure}

\begin{table}[h]
\centering
\caption{\textbf{Model trends are consistent under vibrating-wall dynamics.}
Physical accuracy and visual quality for bidirectional and autoregressive diffusion models on $n=5$ vibrating-wall hard-sphere videos of different lengths $f$.}
\begin{tabular}{ccccccccc}
\toprule
Type & $d_{\text{model}}$ & $n_{\text{layers}}$ & $n_\text{params}$ & TFLOPs & $\Delta x^{(\text{GT})}$ $\downarrow$ & $\Delta x^{(1)}$ $\downarrow$ & $\Delta x^{(5)}$ $\downarrow$ & IoU $\uparrow$ \\
\midrule

\rowcolor{gray!10}
\multicolumn{9}{l}{\textbf{$f = 25$ frames, $n = 5$ balls, vibrating-wall}} \\
\addlinespace[2pt]
Bidir & 768 & 30 & 217M & 10.4 & 1.941 & 0.312 & 1.135 & 0.670 \\
AR    & 768 & 30 & 217M & 7.3  & 1.726 & 0.197 & 0.826 & 0.718 \\

\midrule

\rowcolor{gray!10}
\multicolumn{9}{l}{\textbf{$f = 33$ frames, $n = 5$ balls, vibrating-wall}} \\
\addlinespace[2pt]
Bidir & 768 & 30 & 217M & 13.8 & 2.454 & 0.354 & 1.242 & 0.660 \\
AR    & 768 & 30 & 217M & 10.3 & 2.281 & 0.177 & 0.768 & 0.728 \\

\midrule

\rowcolor{gray!10}
\multicolumn{9}{l}{\textbf{$f = 41$ frames, $n = 5$ balls, vibrating-wall}} \\
\addlinespace[2pt]
Bidir & 768 & 30 & 217M & 17.2 & 2.895 & 0.397 & 1.383 & 0.660 \\
AR    & 768 & 30 & 217M & 13.4 & 2.720 & 0.177 & 0.762 & 0.722 \\

\midrule

\rowcolor{gray!10}
\multicolumn{9}{l}{\textbf{$f = 49$ frames, $n = 5$ balls, vibrating-wall}} \\
\addlinespace[2pt]
Bidir & 768 & 30 & 217M & 20.9 & 3.163 & 0.421 & 1.447 & 0.662 \\
AR    & 768 & 30 & 217M & 16.6 & 3.017 & 0.167 & 0.727 & 0.726 \\

\bottomrule
\end{tabular}
\end{table}

The results in \cref{fig:noisy} are preliminary, and we do not claim that this simple vibrating-wall perturbation fully captures the ambiguity of natural video. Nevertheless, the same qualitative trend persists: bidirectional denoising remains worse on local physical consistency than inference schemes with more effective serial computation. This suggests that the deterministic theory identifies one clean mechanism behind the seriality gap, but the phenomenon may extend beyond the exact assumptions of the theorem.

\section{Ruling Out Alternative Explanations}
\label{app:alternative-explanations}

We test whether the observed advantage of autoregressive generation could be explained by optimization or initialization rather than inference factorization. Specifically, we vary training batch size, repeat training across independent seeds, and fine-tune models initialized from a pretrained Wan checkpoint. In each case, the qualitative autoregressive--bidirectional gap persists.

\subsection{Batch Size}

Diffusion training can be sensitive to batch size, with prior scaling experiments reporting improved generation metrics at larger batch sizes~\citep{Li_2024_CVPR_Scalability}. Because computational constraints limit our main experiments to a batch size $b=64$, we repeat the autoregressive--bidirectional comparison with $b=256$. We keep all other settings fixed, including the learning rate of $2 \times 10^{-4}$, which we do not retune for the larger batch. The results are shown in Figure~\ref{fig:batch_size}. Although the $b=256$ models have higher absolute error, potentially because the learning rate was not retuned, the seriality gap persists at both batch sizes.

\pagebreak

\begin{figure}[h]
  \centering
  \includegraphics[width=0.5\linewidth]{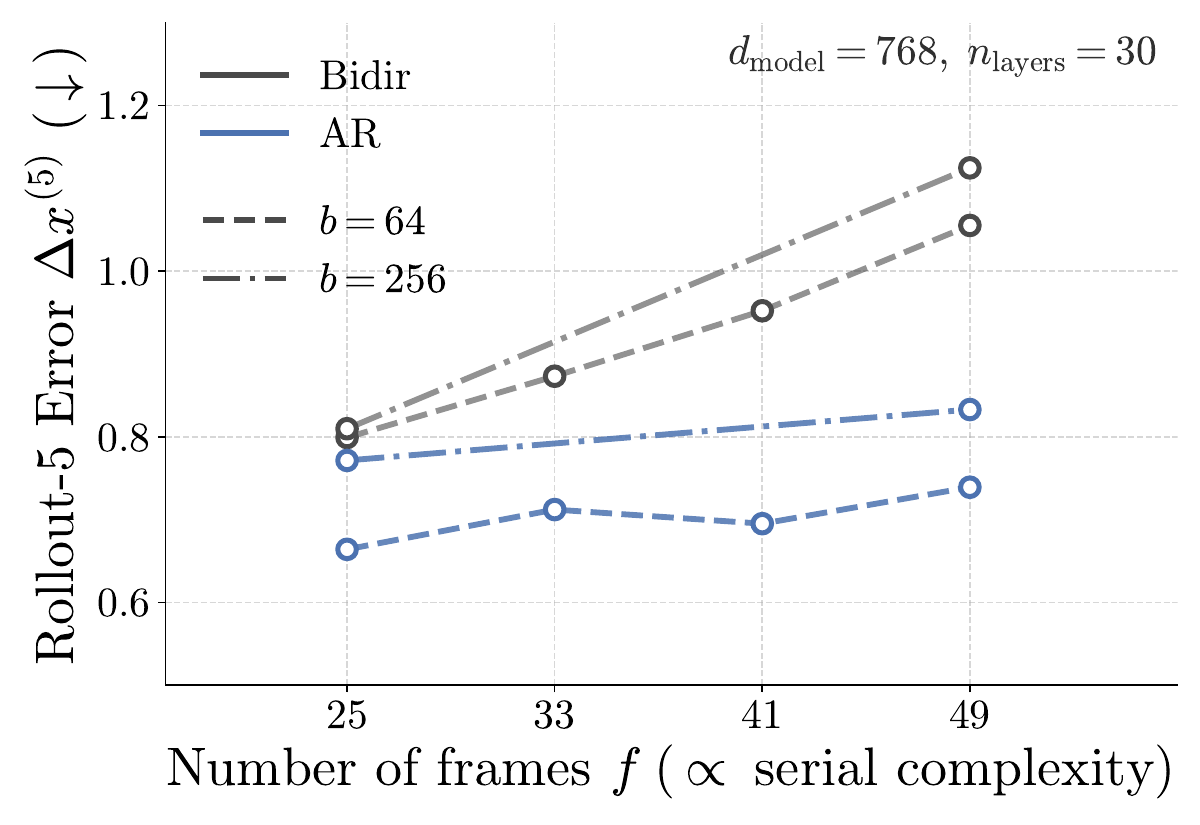}
  \caption{\textbf{Batch-size ablation.}
  Rollout-5 error as a function of video length $f$ for bidirectional and autoregressive models trained with $b\in\{64,256\}$. The $b=256$ models have slightly higher error, possibly because the learning rate was not retuned. At both batch sizes, bidirectional error grows more rapidly with $f$ than autoregressive error, preserving the seriality gap.}
  \label{fig:batch_size}
\end{figure}

\begin{figure}[h]
  \centering
  \includegraphics[width=1.0\linewidth]{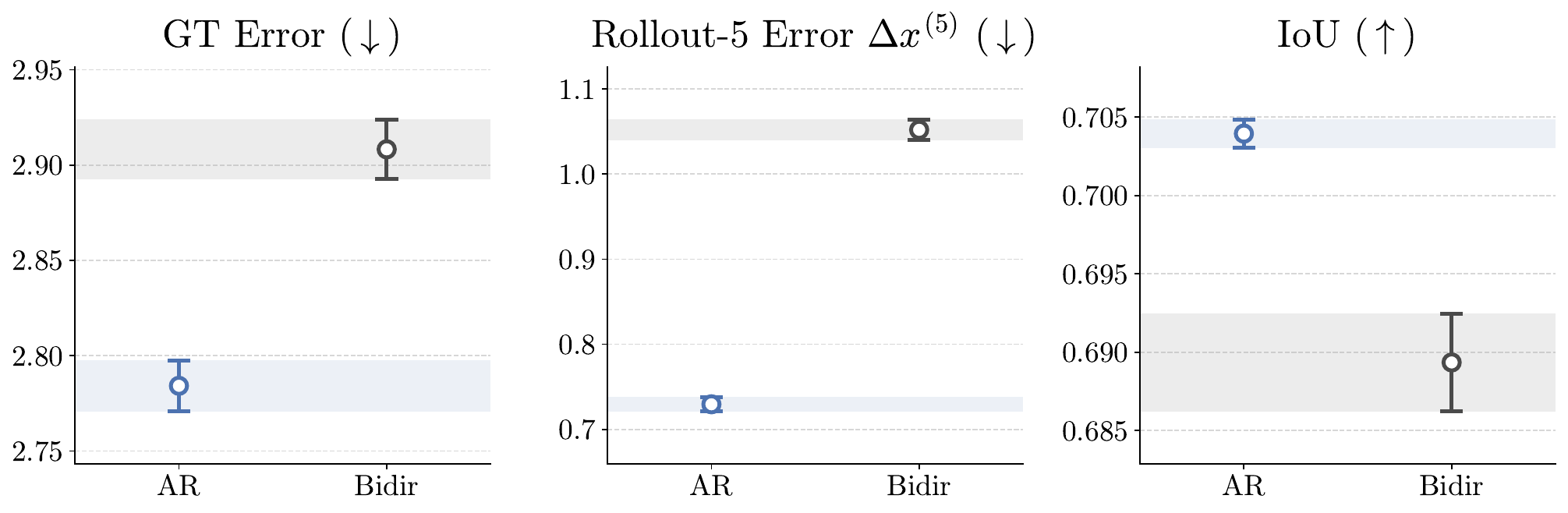}
  \caption{\textbf{Robustness to training-seed variation.}
  Mean global trajectory error $\Delta x^{(\mathrm{GT})}$, rollout-5 error $\Delta x^{(5)}$, and IoU for autoregressive (AR) and bidirectional (Bidir) models across three independent training runs. Error bars and shaded regions denote $\pm1$ standard deviation. AR has lower errors and higher IoU, with non-overlapping intervals.}
  \label{fig:error_bars}
\end{figure}

\begin{figure}[H]
  \centering
  \includegraphics[width=1.0\linewidth]{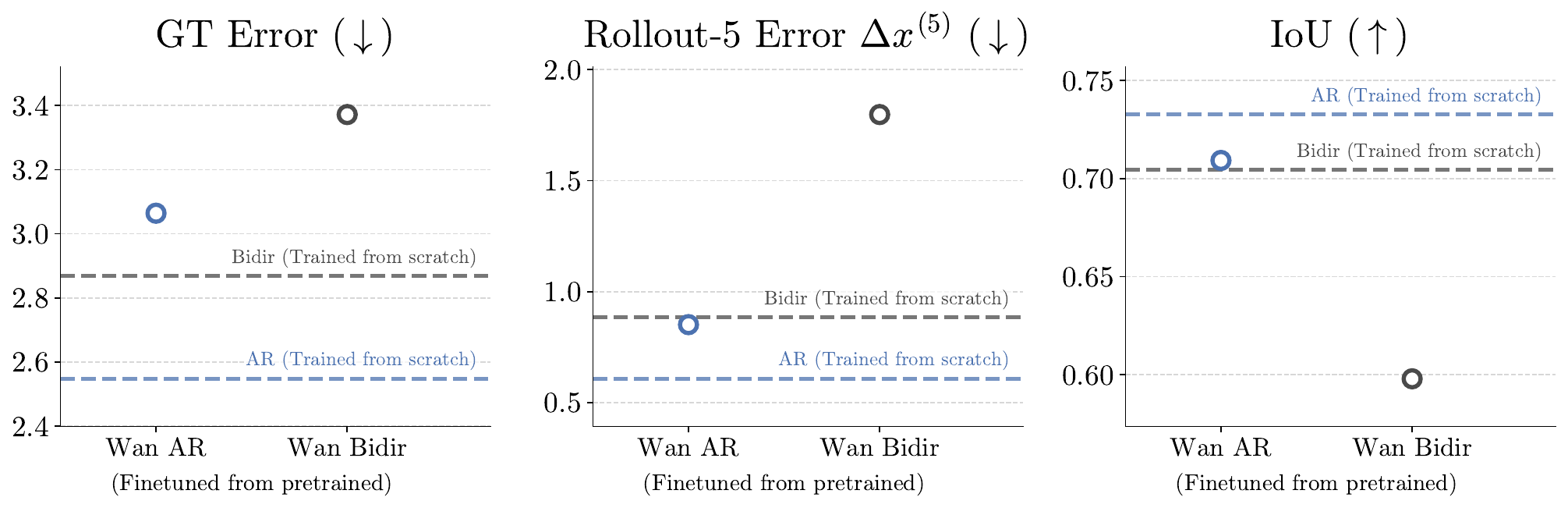}
  \caption{\textbf{Wan initialization preserves the model ordering but lowers absolute performance.}
  Global trajectory error $\Delta x^{(\mathrm{GT})}$, rollout-5 error $\Delta x^{(5)}$, and IoU for models fine-tuned from \texttt{Wan2.1-T2V-1.3B}. Circles show the fine-tuned models, and dashed lines show the corresponding models trained from scratch. AR retains lower errors and higher IoU than Bidir, but both fine-tuned models underperform their trained-from-scratch counterparts, possibly because of a resolution mismatch and the out-of-distribution hard-sphere task.}
  \label{fig:finetuning}
\end{figure}

\pagebreak

\subsection{Training-Seed Variation}
\label{app:lottery}

Independent diffusion training runs can yield different evaluation outcomes due to parameter initialization, minibatch ordering, and sampled training noise~\citep{dufour2026fidlotteryquantifyinghidden}. We therefore train each model with three independent seeds at $d_\text{model}=768$ and $n_\text{layers}=30$, using $f=49$ frames and $n=5$ balls, and evaluate each run on 1,024 examples. The results are shown in Figure~\ref{fig:error_bars}. Across all metrics, the error bars are non-overlapping and well separated, indicating that the autoregressive advantage is robust to training-seed variation rather than the result of a favorable seed lottery.

\subsection{Fine-Tuning from a Pretrained Checkpoint}
\label{app:finetuning}

To test whether pretraining changes the model ordering, we compare our models trained from scratch with variants initialized from the \texttt{Wan-AI/Wan2.1-T2V-1.3B} checkpoint~\citep{wan2025wanopenadvancedlargescale} and fine-tuned on the hard-sphere dataset. We use the same training configuration, changing only the learning rate from $2 \times 10^{-4}$ to $2 \times 10^{-5}$ for fine-tuning. The results are shown in Figure~\ref{fig:finetuning}.

\section{Human Evaluation of Perceptual Plausibility}
\label{app:human-eval}

To validate whether our rollout-$k$ error metric reflects perceptual plausibility, we conduct a lightweight human preference study.
Evaluators used a blind A/B interface, available online\footnote{\url{https://seriality-gap.jdiazchao.com/arena}}.
The interface displays two videos side by side and asks which one looks more physically plausible, with instructions to focus on motion and artifacts rather than sharpness.
We compiled judgments from $21$ volunteer evaluators. No compensation was provided, and the interface did not request personal or demographic information. Each evaluator rated $7$ pairs for each of three comparison groups: AR vs.\ Bidir, AR vs.\ Block-3, and Block-3 vs.\ Bidir, for $21$ pairwise judgments per evaluator and $441$ total judgments.
Video pairs were sampled from examples where the rollout metric indicated a large separation between the two models.

\begin{figure}[h]
  \centering
  \includegraphics[width=\linewidth]{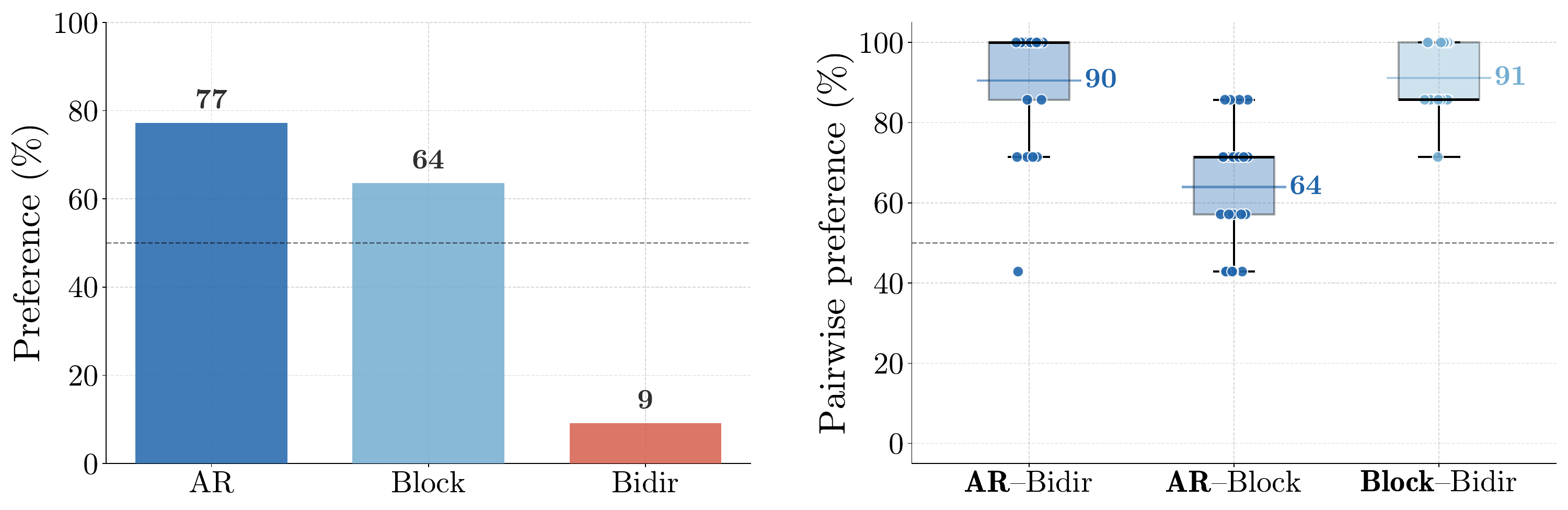}
  \caption{\textbf{Human preference study.}
  Left: aggregate preference rates by model show that AR and Block-3 generations are preferred substantially more often than bidirectional generations. 
Right: pairwise preferences show strong human preference for AR over Bidir and Block-3 over Bidir, while AR is only moderately preferred over Block-3. 
This matches the rollout-error ranking, suggesting that large differences in $\Delta x^{(5)}$ correspond to perceptually salient differences in physical plausibility.}
  \label{fig:user_study}
\end{figure}

The human judgments align with the rollout metric. When $\Delta x^{(5)}$ indicates a large separation, that separation is also perceptible to human evaluators: AR and Block-3 are both strongly preferred over bidirectional generation, whereas the perceptual gap between AR and Block-3 is much smaller. This supports the use of rollout error as a proxy for physically salient generation failures.

\section{Data Synthesis}
\label{data-synth}

Videos are synthesized with a custom NumPy~\cite{Harris_2020} event-driven simulator of frictionless balls in a closed square box.
For ball--ball interactions, the simulator computes the next collision time analytically. For a pair of balls $(i,j)$ with relative displacement $\Delta x = x_j - x_i$, relative velocity $\Delta v = v_j - v_i$, and radii $r_i, r_j$, the collision time $t$ is the smallest nonnegative root of

$$
\|\Delta x + t \Delta v\|_2^2 = (r_i + r_j)^2,
$$

which yields a quadratic equation in $t$. Once $t$ is found, the collision positions follow directly from $x_i(t) = x_i + v_i t$ and $x_j(t) = x_j + v_j t$.

For each rollout, we sample initial ball positions uniformly at random such that no two balls overlap at the first frame, then sample each ball's initial direction uniformly from $[0, 2\pi)$ and speed uniformly from $[8, 10]$. We then simulate a longer trajectory and extract fixed-length clips of length $f$ frames with a sliding window, retaining windows with an overlap of at most $f/2$ to avoid highly similar neighboring clips and increase diversity. Each dataset configuration fixes the number of balls, frame length, box size, ball radius, and rendering resolution.

In the reported experiments, we use $n\in \{1, 5\}$ balls, video lengths $f \in \{25, 33, 41, 49\}$, ball radius $0.7$, world size $10 \times 10$, rendering resolution $128 \times 128$, and $5$ conditioning frames. Training sets contain $20{,}000$ videos. For evaluation, we generate $1{,}024$ base videos of $57$ frames and derive the shorter sets by retaining only the first $f$ frames of each video. This keeps the conditioning frames, and therefore the initial conditions, matched across frame lengths for the same sample, making cross-length comparisons fairer. To avoid trivial trajectories, we retain only windows whose post-conditioning ball--ball collision density,
\[
\frac{\text{number of ball--ball collisions}}
     {\text{active frames}\times\text{number of balls}},
\]
exceeds $0.035$. For $n=5$, this corresponds to at least $5$ collisions when $f=33$ and at least $8$ when $f=49$. For the base $57$-frame evaluation set, we also try to spread samples more evenly over ball--ball collision counts, although the resulting distribution is not perfectly uniform; the shorter trimmed evaluation sets inherit this coverage. Training and evaluation use disjoint random seeds but the same simulator and parameter settings, so evaluation remains in-distribution.

For evaluation, we recover ball positions from generated frames using color-based tracking. 
Each ball is rendered with a distinct RGB color; for each frame and ball, we form a binary mask of pixels whose RGB values lie within a fixed tolerance of that ball's color and compute the mask centroid via image moments. The resulting coordinates are used to compute $\Delta x^{(\mathrm{GT})}$ and $\Delta x^{(k)}$.

The rollout metric $\Delta x^{(k)}$ additionally requires per-ball velocities to seed the simulator. We estimate the velocity at frame $n$ by finite-differencing the recovered centroids, $\hat{v}_n = \hat{x}_n - \hat{x}_{n-1}$. When a collision occurs within the interval $(n-1, n)$, this finite difference straddles the collision and no longer reflects the instantaneous post-collision velocity; in that case, we instead estimate the pre-collision velocity from the preceding interval, $\hat{x}_{n-1} - \hat{x}_{n-2}$, and set $\hat{v}_n$ to the analytic post-collision velocity implied by the hard-sphere collision rule.

\newpage

\section{Attention Masks}
\label{attention-masks}

\begin{figure}[h]
  \centering
  \includegraphics[width=\linewidth]{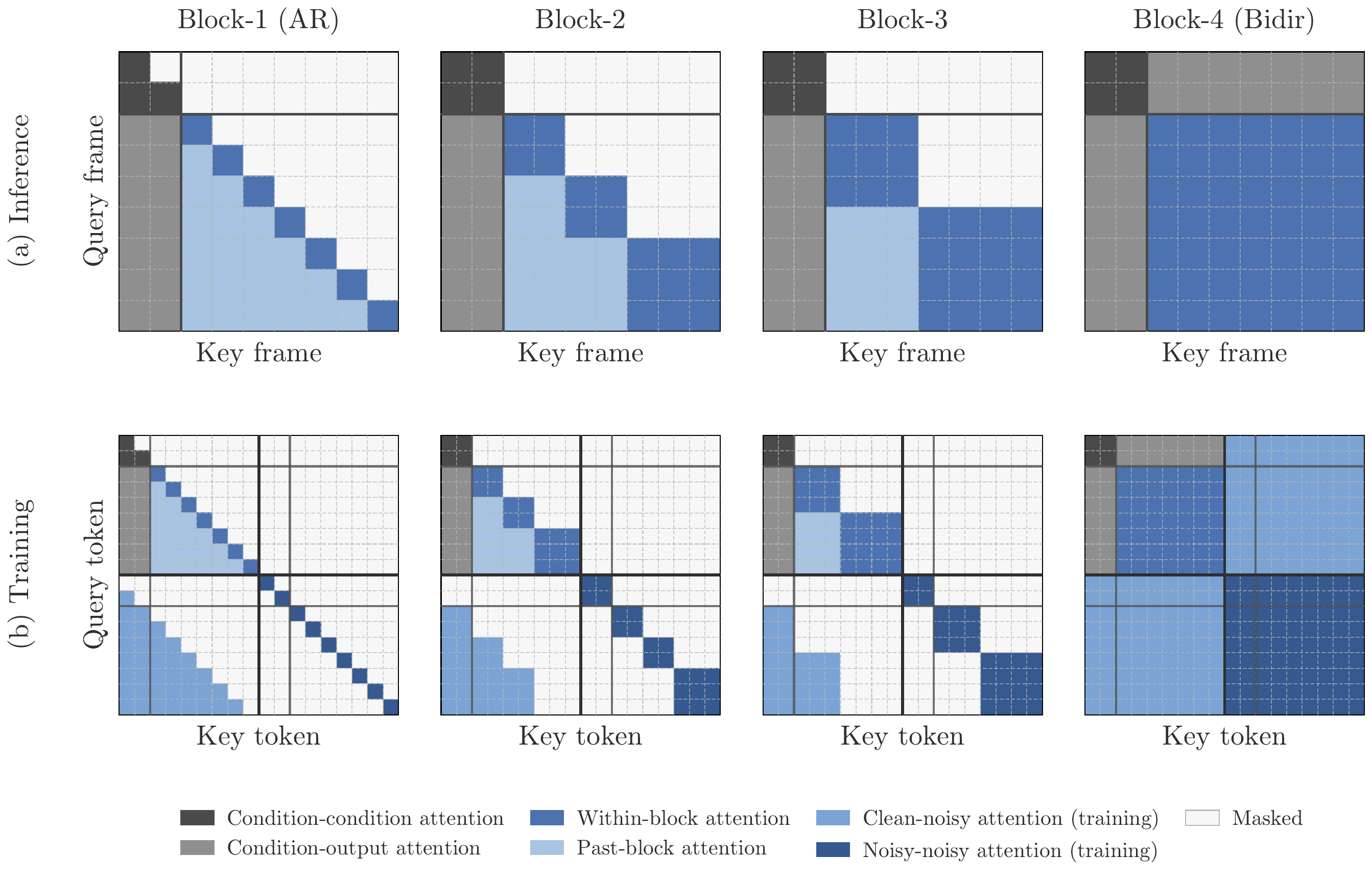}
  \caption{\textbf{Attention masks for blockwise generation.}
  The masks are illustrated for the $f=25$ setting.
  Rows denote query positions and columns denote key/value positions; colored cells indicate permitted attention entries and white cells are masked.
  The top row shows inference-time masks over condition and output frames, while the bottom row shows the corresponding training-time masks over clean and noisy token copies.
  Columns compare autoregressive Block-1 masking, intermediate block-causal masks, and bidirectional masking.
  Thick lines mark condition/output and clean/noisy boundaries; dashed lines show individual frame or token cells.}
  \label{fig:attn-masks}
\end{figure}

\newpage
\end{document}